\documentclass[letterpaper, 10 pt, journal, twoside]{IEEEtran}

\ifCLASSINFOpdf
  \usepackage[pdftex]{graphicx}
  \usepackage[export]{adjustbox}
  \usepackage[hidelinks]{hyperref}
  \hypersetup{pdfauthor=author}
  \usepackage{amsmath,amsthm,amsfonts,amssymb}
  \usepackage{cuted} 
  \usepackage{enumitem}
  \usepackage{balance}
  \usepackage{subcaption}
  \usepackage{cite}
  \usepackage{soul}
  \usepackage[linesnumbered,ruled,vlined]{algorithm2e}
  \usepackage{multirow}
  \usepackage{orcidlink}
  \newtheorem{theorem}{Theorem}
  \newtheorem{remark}{Remark}
  \newtheorem{definition}{Definition}

  \usepackage{array}
  \newcommand{\PreserveBackslash}[1]{\let\temp=\\#1\let\\=\temp}
  \newcolumntype{C}[1]{>{\PreserveBackslash\centering}m{#1}}
  
\else

  \usepackage[dvips]{graphicx}
\fi

\begin{document}
\title{Multi-goal Rapidly Exploring Random Tree\\with Safety and Dynamic Constraints\\for UAV Cooperative Path Planning}

\author{Thu Hang Khuat$^{1*}$ \orcidlink{0009-0003-9611-8678}, Duy-Nam Bui$^{1}$\orcidlink{0009-0001-0837-4360}, Hoa TT. Nguyen$^{2}$,\\Mien L. Trinh$^{3}$\orcidlink{0000-0003-4305-7130}, Minh T. Nguyen$^{2}$\orcidlink{0000-0002-7034-5544}, Manh Duong Phung$^{4}$\orcidlink{0000-0001-5247-6180}~\IEEEmembership{Senior Member,~IEEE}
\thanks{*\textit{Corresponding author}}
\thanks{$^{1}$Thu Hang Khuat and Duy-Nam Bui are with the Vietnam National University, Hanoi, Vietnam. {\tt\footnotesize 23025115@vnu.edu.vn; duynam@ieee.org}}
\thanks{$^{2}$Hoa TT. Nguyen and Minh T. Nguyen are with the Thai Nguyen University of Technology, Thai Nguyen University, Thai Nguyen, Vietnam. {\tt\footnotesize nguyenthituyethoa@tnut.edu.vn; nguyentuanminh@tnut.edu.vn}}
\thanks{$^{3}$Mien L. Trinh is with the University of Transport and Communications, Hanoi, Vietnam. {\tt\footnotesize mientl@utc.edu.vn}}
\thanks{$^{4}$Manh Duong Phung is with the Fulbright University Vietnam, Ho Chi Minh City, Vietnam. {\tt\footnotesize duong.phung@fulbright.edu.vn}}
\thanks{The authors would like to thank the Master, PhD Scholarship Programme (Code: VINIF.2023.Ths.044) of Vingroup Innovation Foundation (VINIF) and the Ministry of Education and Training (Project B2023-GHA-01), Viet Nam, for the support on this work.}
}

\makeatletter


\maketitle

\begin{abstract}
Cooperative path planning is gaining its importance due to the increasing demand on using multiple unmanned aerial vehicles (UAVs) for complex missions. This work addresses the problem by introducing a new algorithm named MultiRRT that extends the rapidly exploring random tree (RRT) to generate paths for a group of UAVs to reach multiple goal locations at the same time. We first derive the dynamics constraint of the UAV and include it in the problem formulation. MultiRRT is then developed, taking into account the cooperative requirements and safe constraints during its path-searching process. The algorithm features two new mechanisms, node reduction and Bezier interpolation, to ensure the feasibility and optimality of the paths generated. Importantly, the interpolated paths are proven to meet the safety and dynamics constraints imposed by obstacles and the UAVs. A number of simulations, comparisons, and experiments have been conducted to evaluate the performance of the proposed approach. The results show that MultiRRT can generate collision-free paths for multiple UAVs to reach their goals with better scores in path length and smoothness metrics than state-of-the-art RRT variants including Theta-RRT, FN-RRT, RRT*, and RRT*-Smart. The generated paths are also tested in practical flights with real UAVs to evaluate their validity for cooperative tasks. The source code of the algorithm is available at {\tt\url{https://github.com/duynamrcv/multi-target_RRT}}
\end{abstract}


%
\IEEEpeerreviewmaketitle

\section{Introduction}
In recent years, the use of multiple unmanned aerial vehicles (UAVs) has emerged as an essential approach for tackling increasingly complex and demanding tasks. By distributing resources and combining sensor data, multiple UAVs can achieve expanded operational areas, faster data collection, and increased mission robustness\mbox{~\cite{10077453,10102336,9628162,do2021formation}}. These capabilities are particularly advantageous in many applications such as search and rescue operations, environmental monitoring, disaster response, and precision agriculture, where extensive coverage and real-time data collection from multiple perspectives are important. To obtain these capabilities, the UAVs need a robust and effective cooperative path planning system to guide them through complex environments, avoid collisions, and optimize their routes.

In cooperative path planning, the goal is to find a feasible path for each UAV from its start to its final location, fulfilling several requirements and constraints~\cite{Xu2020,Xu2023}. The requirements involve task allocation and optimality, such as arriving at certain locations at the same time,  minimizing the path length and energy consumption, and maintaining communication\mbox{\cite{10134570,10234672,9703683,10570813}}. The constraints relate to UAV dynamics and safety such as turning angles and distance to obstacles. Approaches to this problem can be categorized into four primary groups including mathematical programming, swarm intelligence, artificial potential field, and graph-based planning methods.

Mathematical programming addresses the path planning problem by modeling it as an optimization where the aim is to minimize or maximize an objective function subject to a set of constraints. Programming techniques such as mixed integer programming (MIP), nonlinear programming (NP), and dynamic programming (DP) are then used to solve the problem~\cite{10225271,Ioan2021, 10149375}. In~\cite{10225271}, the coverage path planning problem for multiple robots is reduced to the min-max rooted tree cover problem. An MIP model is then proposed to solve that problem to find the paths. In~\cite{9170782}, a set of mixed integer linear programming models is introduced to find the paths for multiple robots to capture a moving target within a given time limit. In~\cite{1428837}, a neuro-dynamic programming algorithm is proposed to generate paths that allow the UAVs to avoid threat zones and arrive at the targets at the same time from different directions. These approaches provide suboptimal solutions and offer flexibility to handle multiple constraints. However, their performance depends on the definition of objective functions, which are often simplified representations of the real-world scenario. In addition, the computational cost of mathematical programming techniques grows exponentially as the scale of the planning problem increases.

Similar to mathematical programming, swarm intelligence techniques also formulate path planning as an optimization problem. However, they use nature-inspired algorithms such as the genetic algorithm (GA) and particle swarm optimization (PSO) to solve it. These techniques utilize swarm intelligence, such as social and collective behaviors, to better explore and exploit the search space for the optimal solutions. In~\cite{cao2019multi}, the GA is used to solve the cooperative reconnaissance path planning for multiple UAVs taking off from different bases. In~\cite{ZainACO5727}, maximum-minimum ant colony optimization and differential evolution are combined to generate collision-free paths for multiple UAVs. An improved grey wolf optimizer is introduced in~\cite{yao2016multi} to solve the trajectory planning subject to various constraints. Other algorithms such as PSO and ant colony optimization (ACO) are also used in ~\cite{Xu2023,10417512,Das2020} to address the cooperative path planning problem with different constraints and requirements. The main advantages of nature-inspired algorithms lie in their simple structure and fast generation of solutions~\cite{yao2016multi}. The quality of the solutions, however, depends on the initialization, population size, and mechanism to avoid local optimum.

Artificial potential field (APF) is another approach that addresses the path planning problem by simulating virtual forces that influence the UAV~\cite{9234396, 9591322}. In this method, the goal generates attractive forces while obstacles create repulsive forces so that they together drive the robot through the environment like a particle moving in a force field. In particular, a rotating potential field is used in~\cite{9538804} to generate paths for a multi-UAV system that can avoid local minimum. In~\cite{9143127}, an improved APF is introduced where the k-means algorithm is used to optimize attractive forces. The paths generated by APF, however, are not optimal since they do not include a cost function for evaluation. Local optimum is also another issue as the attractive and repulsive forces can be equal at certain potential points~\cite{9591322}. 

Graph-based path planning is another direction that represents the environment as a network of nodes and edges, where nodes represent possible positions and edges represent viable paths between them \cite{graph2020}. This representation allows algorithms such as probabilistic roadmap (PRM) \cite{5406221} and rapidly-exploring random tree (RRT) \cite{02783640122067453} to determine the route from a starting point to a destination. In \cite{yan2013path}, the PRM is used to generate paths for UAVs in complex 3D environments in which the A* is utilized to find feasible paths. In \cite{kothari2013probabilistically}, the RRT is employed to generate flyable paths in real time for UAVs. While the graph-based approach is fast and efficient, the paths generated are not optimized due to its random expansion mechanism. Some improvements have been introduced to address this issue such as RRT* \cite{Karaman2011} and RRT*-Smart \cite{Nasir2013}. They however do not consider safety and dynamic constraints which are critical to generate feasible paths for UAVs. 

Recently, deep reinforcement learning has been used to model environmental dynamics and optimize UAV trajectories\mbox{\cite{10060414}}. For instance, a deep Q network is introduced in\mbox{\cite{10556816}} to control the movements of the UAVs in adaption to the terrestrial structure for Peer-to-Peer (P2P) communication between UAVs. However, this method requires significant computational resources for both training and real-time decision-making and faces difficulties in generalizing learned policies across different environments with multiple moving UAVs. In addition, multiple UAVs also bring new challenges related to their heterogeneous capability, task allocation, and time coordination. The high mobility of UAVs causes the topology of multi-UAVs to become dynamic, which makes the UAVs' communication connections sparse and intermittent~\cite{luo2023path,Minh10212603}. Therefore, incorporating time coordination, dynamic constraints, and path smoothing into path planning to increase its feasibility and efficiency is a problem that is worth further investigation.

In this work, we address the cooperative path planning problem by introducing a new algorithm, named MultiRRT, to generate paths for multiple UAVs cooperated in a complex environment. We first formulate the problem by defining the task requirements, UAV models, and dynamics constraints. MultiRRT is then introduced to solve the problem by generating paths for multiple UAVs. Several new mechanisms have been included in the algorithm to optimize the paths. Our contributions are threefold:
\begin{enumerate}[label=\roman*.]
    \item Derivation of the dynamic constraint of the quadrotor UAV, which is essential to generate flyable paths. 
    \item Proposal of MultiRRT that incorporates safety and dynamics constraints during its branch expansion to find paths to multiple goal locations. The algorithm also features a new node reduction mechanism to optimize the path length.
    \item Introduction of a new Bezier interpolation technique to smooth the flight paths. The interpolated paths are proven to be collision-free and satisfy the dynamic constraint.
\end{enumerate}

The rest of this paper is organized as follows. Section~\ref{sec:problem} describes the cooperative path planning problem. Section~\ref{sec:model} introduces the UAV model and the calculation of its dynamic constraint. Section~\ref{sec:propose} presents the proposed path planning algorithm for multiple UAVs. Simulations, comparisons, and experiments are presented in Section~\ref{sec:results}. The paper ends with a conclusion drawn in Section~\ref{sec:conclusion}.
\section{Problem Formulation}\label{sec:problem}
\begin{figure}
    \centering
    \includegraphics[width=0.48\textwidth]{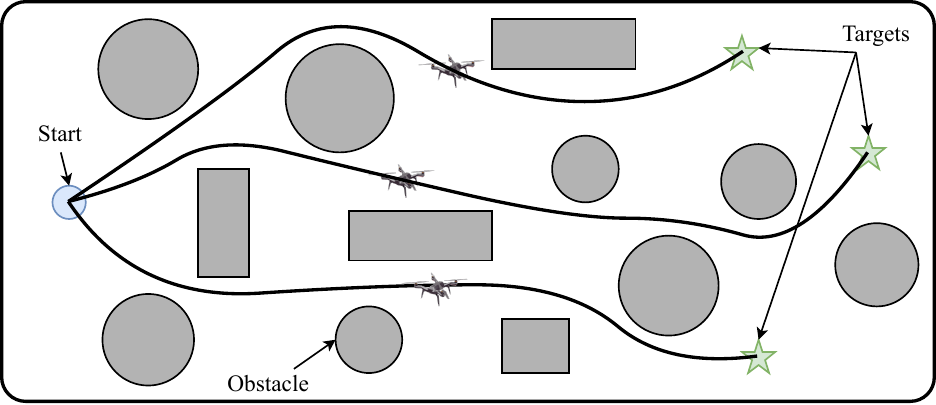}
    \caption{A cooperative path planning task where multiple UAVs need to reach different target locations at the same time}
    \label{fig:problem}
\end{figure}

This study considers a group of UAVs that must traverse a complex environment to reach different target locations and complete tasks such as disaster rescue or emergency network establishment, as shown in Figure~\ref{fig:problem}. Assuming that there are $n$ target locations, the task involves finding $n$ paths that meet the following requirements: 

\begin{enumerate}[label=\roman*.]
  \item Guide the UAVs to simultaneously arrive $n$ pre-defined locations; 
  \item Satisfy constraints imposed by the UAV's dynamics;
  \item Ensure that there are no collisions among the UAVs or between the UAVs and their operating environment. 
\end{enumerate}

The first requirement allows the UAVs to fulfill their cooperative task, while the second and third ones maintain their safety and efficiency. In our approach, the UAVs are modeled like Dubins vehicles flying at a fixed altitude with a constant forward speed, a widely adopted model for UAVs~\cite{5427036, manyam2017tightly, YAO2017217}. In this model, the collision among UAVs can be avoided by simply choosing a different flight altitude for each UAV, i.e., $h_i \neq h_q$, $\forall i\neq q \in \{1,...,n\} $, where $h_i$ is the altitude of UAV $i$. Let $v_i$ be the speed of UAV $i$ and $L_i$ be the length of path $\mathcal{P}_i$. The cooperation in time among the UAVs then can be obtained by choosing $v_i$ so that it is inversely proportional to the path length, i.e.,
\begin{equation}
    \dfrac{L_i}{v_i} =\dfrac{L_{i+1}}{v_{i+1}}\quad\forall i={1,...,n-1}.
\end{equation}

To include dynamic constraints, this work presents path $\mathcal{P}_i$ as a set of $N_i$ line segments consisting of $N_i+1$ points $P_{ij}$. A flyable path then needs to have its turning angles staying within the maneuverable limits of the UAV. Let $\gamma_{ij}$ be the turning angle at $P_{ij}$, which is defined by the angle between two vectors $\overrightarrow{P_{i,j-1}P_{ij}}$ and $\overrightarrow{P_{ij}P_{i,j+1}}$, as illustrated in Figure~\ref{fig:angle}. It is given by

\begin{equation}    \gamma_{ij}=\cos^{-1}\left(\dfrac{\overrightarrow{P_{i,j-1}P_{ij}}.\overrightarrow{P_{ij}P_{i,j+1}}}{\left\Vert \overrightarrow{P_{i,j-1}P_{ij}}\right\Vert \left\Vert \overrightarrow{P_{ij}P_{i,j+1}}\right\Vert }\right).
    \label{eqn:angle}
\end{equation}
Denote $\gamma_\text{max}$ as the maximum turning angle of the UAV. A flyable path then needs to meet the following condition:
\begin{equation}
    \gamma_{ij}  \leq \gamma_\text{max}, \quad\forall j \in \{2,...,N\}.
    \label{eqn:constraint}
\end{equation}
The value of $\gamma_\text{max}$ can be computed based on the UAV dynamics presented in the next section.


\section{UAV dynamic model and constraint analysis}\label{sec:model}

\begin{figure}
\centering
\begin{subfigure}[b]{0.21\textwidth}
    \centering
    \includegraphics[width=\textwidth]{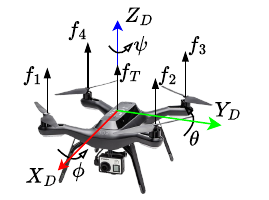}
    \caption{The UAV model}
    \label{fig:vehicle}
\end{subfigure}
\begin{subfigure}[b]{0.27\textwidth}
    \centering
    \includegraphics[width=\textwidth]{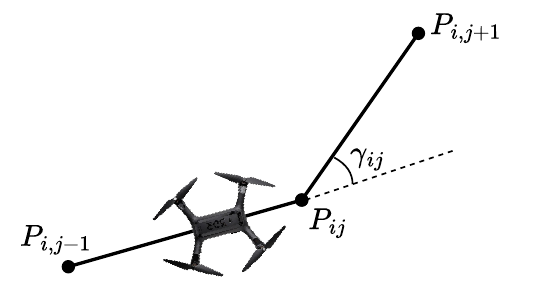}
    \caption{Turning angle}
    \label{fig:angle}
\end{subfigure}
\caption{UAV model and its turning angle}
\end{figure}

\subsection{Dynamic model}
The UAVs used in this work are quadcopters with two pairs of propellers rotating in opposite directions, as depicted in Figure~\ref{fig:vehicle}. Let $f_l$, $l=\left\{1,2,3,4\right\}$, be the thrust generated by propeller $l$. According to~\cite{Wang2016}, $f_l$ is proportional to the square of rotational speed $\Omega_l$ of motor $l$, i.e.,
\begin{equation}
    f_l=c_t\Omega_l^2\quad\forall l={1,2,3,4},
\end{equation}
where $c_t$ is the thrust coefficient. The total thrust is then given by
\begin{equation}
    f_T=\sum_{k=1}^{4}f_{l}.
\end{equation}

Since $f_T$ acts along the $Z_D$ axis, the thrust vector in the body frame is given by $f_D = \left[0,0,f_T\right]^T$. Apart from the thrust, the UAV is also influenced by drag force, $f_d$, caused by air friction. It is proportional to the UAV's linear velocities, i.e., $f_d = c_f\left[\dot{x},\dot{y},\dot{z}\right]^T$, where $c_f$ is the friction coefficient. The motion of the UAV in the inertial frame is then driven by the gravitational, thrust, and drag forces. Let $R$ be the rotation matrix mapping the body frame to the inertial frame, then
\begin{equation}
    R_{xyz}=R_{z}\left(\psi\right)R_{y}\left(\theta\right)R_{x}\left(\phi\right)
\end{equation}
with
\begin{equation}
\begin{aligned}
    R_{z}\left(\psi\right)&=\left[\begin{array}{ccc}
\cos\psi & -\sin\psi & 0\\
\sin\psi & \cos\psi & 0\\
0 & 0 & 1
\end{array}\right]\\
    R_{y}\left(\theta\right)&=\left[\begin{array}{ccc}
\cos\theta & 0 & \sin\theta\\
0 & 1 & 0\\
-\sin\theta & 0 & \cos\theta
\end{array}\right]\\
    R_{x}\left(\phi\right)&=\left[\begin{array}{ccc}
1 & 0 & 0\\
0 & \cos\phi & -\sin\phi\\
0 & \sin\phi & \cos\phi
\end{array}\right]
\end{aligned}
\end{equation}
where $\phi$, $\theta$, and $\psi$ are respectively the roll, pitch, and yaw angles, and $R_{x}$, $R_{y}$, and $R_{z}$ are respectively the rotation matrices about the $x$, $y$, and $z$ axes. The dynamic equations describing the linear motion of the UAV are then given by
\begin{equation}
    \left[\begin{array}{c}
\ddot x\\
\ddot y\\
\ddot z
\end{array}\right] = \frac{1}{m}\left(R_{xyz}f_D - f_d - \left[\begin{array}{c}
0\\
0\\
mg
\end{array}\right]\right) ,
\label{eqn:dynamics}
\end{equation}
where $m$ is the mass of the UAV and $g$ is the gravitational acceleration.

\subsection{Dynamic constraint analysis} \label{sec:UAVconstraint}
Assume that at time $t$, the UAV is flying along $X_D$ at velocity $v_x$ with pitching angle $\theta$. The UAV then tries to turn an angle $\gamma_\text{max}$ in direction $Y_D$ by carrying out a rotation of an angle $\phi$ about the $X_D$ axis. The calculation of $\gamma_\text{max}$ then can be derived as follows.

Let $f_I = [f_{Ix},f_{Iy},f_{Iz}]^T$ be the force acting on the UAV in the inertial frame. This force relates to $f_D$ by the rotations of $\phi$ and $\theta$ about the $X_D$ and $Y_D$ axes as follows: 

\begin{equation}
    f_I = R_{y}\left(\theta\right)R_{x}\left(\phi\right)f_D= \left[\begin{array}{c}
f_T \cos\phi \sin\theta\\
-f_T \sin\phi \\
f_T \cos\phi \cos\theta
\end{array}\right].
\label{eqn:force_inertial}
\end{equation}

Since the UAV flies at constant speed $v_x$, thrust force $f_{Ix}$ is balanced with drag force $f_{dx}$. From \eqref{eqn:force_inertial}, we get 
\begin{equation}
    f_{Ix}=f_{T}\cos\phi\sin\theta=c_fv_{x}.
    \label{eqn:fx}
\end{equation}

When the UAV turns, $f_{Iy}$ will generate centripetal acceleration, $a_y$, in $Y_B$ direction and $v_x$ will become the tangential velocity, i.e.,
\begin{equation}
    f_{Iy}=-ma_y=-\frac{mv_x^2}{R},
    \label{eqn:ay}
\end{equation}
where $R$ is the turning radius and the negative sign implies that the direction of $a_y$ is opposite of the $Y_B$ axis. Substituting $f_{Iy}$ from (\ref{eqn:force_inertial}) to (\ref{eqn:ay}) gives
\begin{equation}
    f_T\sin\phi=\frac{mv_x^2}{R}.
    \label{eqn:phi_constraint}
\end{equation}

As the UAV flies at a fixed altitude, the force $f_{Iz}$ acting in $z$ direction is balanced with the gravity. We have
\begin{equation}
  f_{Iz}= f_{T}\cos\phi\cos\theta= mg.
    \label{eqn:fz}
\end{equation}

From (\ref{eqn:fx}) and (\ref{eqn:fz}), pitch angle $\theta$ can be computed as follows:
\begin{equation}
    \theta=\tan^{-1}\left(\frac{c_fv_x}{mg}\right).
    \label{eqn:theta}
\end{equation}

Substituting (\ref{eqn:theta}) into (\ref{eqn:fz}) gives
\begin{equation}
    \phi=\cos^{-1}\left(\frac{c_fv_x}{f_T\sin\theta}\right).
    \label{eqn:costheta}
\end{equation}

On the other hand, from (\ref{eqn:theta}), we get
\begin{equation}
    \sin\theta=\dfrac{1}{\sqrt{\dfrac{1}{\tan^2\theta}+1}}=\dfrac{c_fv_x}{\sqrt{c_f^2v_x^2 + m^2g^2}}.
    \label{eqn:sintheta}
\end{equation}

Rearranging (\ref{eqn:phi_constraint}) and substituting (\ref{eqn:costheta}) and (\ref{eqn:sintheta}) into it give

\begin{equation}
    R=\frac{mv_x^2}{f_T\sqrt{1 - \cos^2\phi}}=\frac{mv_x^2}{\sqrt{f_T^2 -c_f^2v_x^2-m^2g^2 }}.
    \label{eqn:radius}
\end{equation}

As the turning angle $\gamma$ is inversely proportional to the turning radius $R$ and $v_x$ is constant, we have

\begin{equation}
    \gamma_\text{max}=\dfrac{1}{R_\text{min}}= \dfrac{\sqrt{f_\text{Tmax}^2-c_f^{2}v_x^{2}-m^{2}g^{2}}}{mv_x^{2}} \label{eqn:max_angle} ,
\end{equation}
where $f_\text{Tmax}$ is the maximum thrust force. The maximum turning angle $\gamma_\text{max}$ defines the limit at which a controller can adjust the UAV's direction to follow the planned path. The path planning algorithm therefore must take into account this constraint when generating paths to ensure they are feasible for the UAV to follow in practice.
\section{MultiRRT for cooperative path planning}\label{sec:propose}

This section presents our approach to address the cooperative path planning problem by extending the RRT for multiple goals.

\subsection{Multi-goal RRT with dynamic constraint}\label{sec:angle}

\begin{figure}
    \centering
    \includegraphics[width=0.4\textwidth]{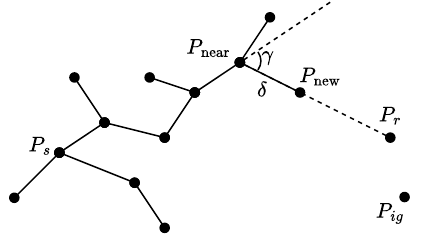}
    \caption{Illustration of the enhanced RRT algorithm with the constraint on the turning angle}
    \label{fig:RRTalg}
\end{figure}

Denote $P_s$ as the start location of the UAVs and $P_{ig}$ as the goal location for UAV $i$. The path for the $i^{th}$ UAV is represented by a set of nodes $\mathcal{P}_i=\{P_{ij}\}$, $\forall j \in \{1,...,N_i+1\}$. Denote $\mathcal{G}$ as the RRT tree that contains all paths $\mathcal{P}_i$, i.e., $\mathcal{P}_i\subset\mathcal{G}$, $\forall i\in\left\{1,...,n\right\}$. $\mathcal{G}$ is formed through an expansion process that involves random nodes. The expansion starts with a node $P_r$ chosen randomly within the environment, as illustrated in Figure~\ref{fig:RRTalg}. Let $P_\text{near} \in \mathcal{G}$ be the node closest to $P_{r}$. A new node $P_\text{new}$ is then created to lie on the vector $\overrightarrow{P_\text{near}P_{r}}$ with a predefined distance $\delta$ as follows
\begin{equation}
    \begin{cases}
\overrightarrow{P_\text{near}P_{r}}=\lambda\overrightarrow{P_\text{near}P_\text{new}}\\
\left\Vert \overrightarrow{P_\text{near}P_\text{new}}\right\Vert =\min\left\{\left\Vert \overrightarrow{P_\text{near}P_r}\right\Vert,\delta\right\} ,
\end{cases}
\label{eqn:steer}
\end{equation}
where $\lambda\neq0$ is a scalar. In our work, $P_\text{new}$ is only added to tree $\mathcal{G}$ if it satisfies two conditions: (i) the line segment connecting this node to its nearest neighbor ($P_\text{near}$) is collision-free; and (ii) the turning angle of the resulting path segment does not exceed the maximum turning angle ($\gamma_\text{max}$). The first condition guarantees safety while the second ensures that the generated paths meet the UAVs' dynamic constraints. In collision checking, the UAV's size is considered by expanding the boundaries of obstacles in the environment by an amount equal to the UAV's radius. For the turning angle, assume $P_\text{new}$ is connected to node $P_k\in\mathcal{G}$ and $P_{k-1}$ is the parent node of $P_k$. The turning angle $\gamma$ at $P_k$ is then the angle between $\overrightarrow{P_{k-1}P_{k}}$ and $\overrightarrow{P_kP_\text{new}}$. This angle needs to meet constraint \eqref{eqn:constraint} for $P_\text{new}$ to be added to $\mathcal{G}$, i.e.,
\begin{equation}
\begin{aligned}
    \gamma&=\cos^{-1}\left(\dfrac{\overrightarrow{P_{k-1}P_k}.\overrightarrow{P_kP_\text{new}}}{\left\Vert \overrightarrow{P_{k-1}P_k}\right\Vert \left\Vert \overrightarrow{P_kP_\text{new}}\right\Vert }\right)\leq\gamma_\text{max}.
\end{aligned}
\label{eqn:gamma}
\end{equation}

Algorithm~\ref{alg:rrt} presents implementation of the algorithm, where the tree expansion only stops if all goal locations $P_{ig}$ are added to the tree. The paths to all goals are then generated by tracing back the tree.

\begin{algorithm}
\tcc*[h]{Initialization}\\
$\mathcal{P}_i\leftarrow\varnothing$, $\forall i \in\left\{1,...,n\right\}$\;
$\mathcal{G}\leftarrow P_s$\;
count $\leftarrow$ 0\tcc*[r]{number of paths}
\tcc*[h]{Path generation via tree expansion}\\

\While{count $\neq$ n}
{
    Generate random node $P_r$\;
    Find node $P_\text{near}\in\mathcal{G}$ nearest to $P_r$\;
    Generate node $P_\text{new}$ based on equation \eqref{eqn:steer}\;
    \If{$\overrightarrow{P_\text{near}P_\text{new}}$ is collision-free and satisfies constraint \eqref{eqn:gamma}}
    {
        $\mathcal{G}\leftarrow P_\text{new}$\;
        \For{$i\leftarrow$1 to n}
        {
            \If{$\overrightarrow{P_\text{new}P_{ig}}$ is collision-free and satisfies constraint \eqref{eqn:gamma}}
            {
                $\mathcal{G}\leftarrow P_{ig}$\;
                Generate path $\mathcal{P}_i$ by tracing back tree $\mathcal{G}$ from end node $P_{ig}$\;
                count $\leftarrow$ count + 1\;
            }
        }
    }
}
\caption{Pseudo code of the enhanced RRT algorithm with dynamic constraint}
\label{alg:rrt}
\end{algorithm}

\subsection{Node reduction mechanism} \label{sec:remove_node}

\begin{figure}
    \centering
    \includegraphics[width=0.45\textwidth]{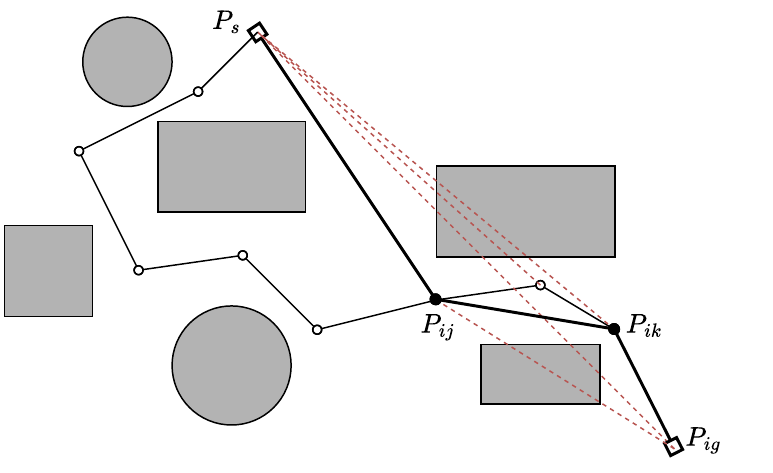}
    \caption{Illustration of redundant node reduction using the triangle inequality}
    \label{fig:reduce}
\end{figure}

The paths generated by RRT are not optimal due to its random expansion fashion. We address this using a node reduction algorithm as illustrated in Figure~\ref{fig:reduce}. First, the algorithm checks if there exists a node $P_{ij}$ nearby the start node $P_s$ and a node $P_{ik}$ nearby the goal $P_{ig}$. It then evaluates if the path segment $P_{ij}P_{ik}$ is collision-free and satisfies constraint \eqref{eqn:gamma}. The algorithm finally connects $P_{ij}$ to $P_{ik}$ and removes the nodes between them to form a shorter path. Denote $P_{i,\text{tail}}$ as the end node of $\mathcal{P}_i$, the node reduction algorithm is presented in Algorithm~\ref{alg:remove_node}.

\begin{algorithm}
\SetAlgoLined
\tcc*[h]{Initialization}\\
Create a copy path $\mathcal{P}'_i\leftarrow \mathcal{P}_i$\;
Reset $\mathcal{P}_i\leftarrow\varnothing$ then $\mathcal{P}_i\leftarrow P_s$\;

\tcc*[h]{Remove redundant nodes}\\
\While{$P_{i,\text{tail}}\neq P_{ig}$}
{
    \ForEach{$P_{ij}\in\mathcal{P}'_i$}
    {
        \If{$\overrightarrow{P_{i,\text{tail}}P_{ij}}$ is collision-free and satisfy constraint \eqref{eqn:gamma}}
        {
            $\mathcal{P}_i\leftarrow P_{ij}$\;
            break\;
        }
    }
}
\caption{Pseudo code of the node reduction mechanism}
\label{alg:remove_node}
\end{algorithm}

\subsection{Path smoothing using Bezier curves}\label{sec:spline}
The paths generated by RTT include a set of interconnected line segments that need to be smoothed for the UAVs to follow. We conduct this using an interpolation technique in which a Bezier curve is fitted to each node of the flight path. First, a safe zone is defined at each node as follows.

\begin{definition}
A safe zone is a circular area around a node whose radius is the distance from that node to its nearest obstacle.
\label{def:free_zone}
\end{definition}

\begin{figure}
    \centering
    \includegraphics[width=0.4\textwidth]{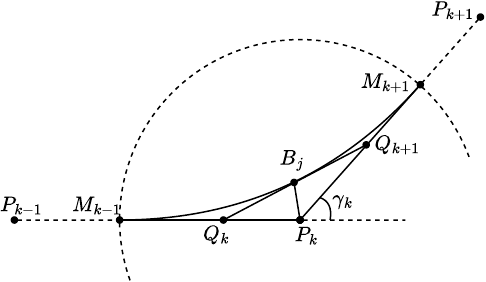}
    \caption{The safe zone and its intersection with path segments}
    \label{fig:bezier}
\end{figure}

Consider the safe zone at node $P_k$ with radius $R_k$ that intersects path segments $P_{k-1}P_{k}$ and $P_{k}P_{k+1}$ at $M_{k-1}$ and $M_{k+1}$, respectively, as depicted in Figure~\ref{fig:bezier}. Denote $M_{k} \equiv P_k$, the two first-order curves $H_1(\tau)$ and $H_2(\tau)$ can be created from $M_{k-1}$, $M_k$, and $M_{k+1}$ as follows

\begin{equation}
\begin{aligned}
    H_{1}\left(\tau\right)&=M_{k-1}+\tau\left(M_{k}-M_{k-1}\right) \quad 0\leq \tau\leq 1\\
    H_{2}\left(\tau\right)&=M_{k}+\tau\left(M_{k+1}-M_{k}\right) \quad \quad0\leq \tau\leq 1 .
\end{aligned}
    \label{eqn:H}
\end{equation}
The quadratic Bezier curve $B_k(\tau)$ is then created from $H_1(\tau)$ and $H_2(\tau)$ as
\begin{equation}
\begin{aligned}
B_k\left(\tau\right) &= \tau H_{1}\left(\tau\right) + (1-\tau)H_{2}\left(\tau\right)\\
&= (1-\tau)^2M_{k-1}+2(1-\tau)\tau M_{k}+\tau^2M_{k+1}.
\end{aligned}
\label{eqn:bezier}
\end{equation}

\begin{theorem} \label{the:1}
The second order Bezier curve $B_k\left(\tau\right)$ defined in \eqref{eqn:bezier} always lies within the safe zone at point $P_k$. That is, the distance from any point $B_j$ on $B_k\left(\tau\right)$ to $P_k$ is always less than or equal to the radius of the safe zone at $P_k$:
\begin{equation}
\left\Vert P_{k}B_j\right\Vert \leq R_k \quad \forall B_j \in B_k(\tau)
\end{equation}
\end{theorem}

\begin{proof}
Consider point $B_j \in B_k(\tau)$ and denote $Q_{k}$ and $Q_{k+1}$ respectively as the intersections of the tangent line of $B_k(\tau)$ at that point with $P_{k-1}P_k$ and $P_{k}P_{k+1}$, as depicted in Figure~\ref{fig:bezier}. From triangle $\triangle Q_kP_kQ_{k+1}$, we have
\begin{equation}
\left\Vert P_{k}B_j\right\Vert \leq \max\left\{\left\Vert P_{k}Q_k\right\Vert,\left\Vert P_{k}Q_{k+1}\right\Vert\right\}.
\label{eq:dist1}
\end{equation}

Since $Q_k \in [M_{k-1},P_{k}]$ and $Q_{k+1} \in [P_{k},M_{k+1}]$, we get
\begin{equation}
    \begin{aligned}
    &\left\Vert P_kQ_{k+1}\right\Vert \leq \left\Vert P_{k}M_{k+1}\right\Vert\\
    &\left\Vert P_kQ_k\right\Vert \leq \left\Vert P_{k}M_{k-1}\right\Vert
    \end{aligned}
    \label{eq:dist2}
\end{equation}

According to (\ref{eq:dist1}) and (\ref{eq:dist2}), and note that $\left\Vert P_{k}M_{k-1}\right\Vert = \left\Vert P_{k}M_{k+1}\right\Vert = R_k$, we have
\begin{equation}
\left\Vert P_{k}B_j\right\Vert \leq R_k \quad \forall B_j \in B_k(\tau).
\label{eq:dist3}
\end{equation}

\end{proof}

\begin{figure}
    \centering
    \includegraphics[width=0.48\textwidth]{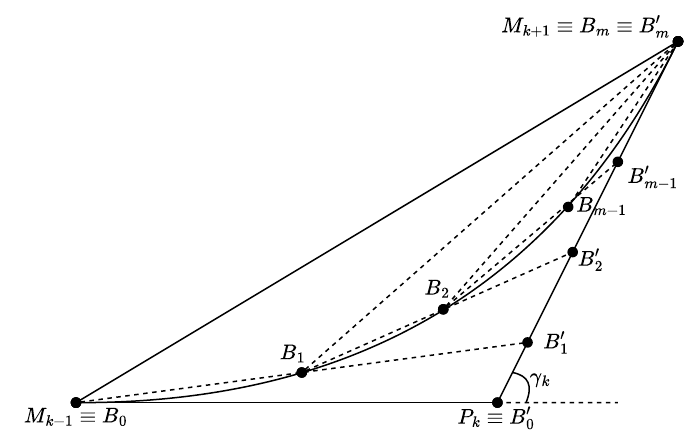}
    \caption{Turning angles at certain points on the
interpolated Bezier curve}
    \label{fig:bezier_dem}
\end{figure}

Next, we prove that the turning angle at any point on the interpolated Bezier curve meets the dynamic constraint. Let $B=\{B_0, B_1, ..., B_m\}$ be a set of points on $B_k(\tau)$ so that $B_0 \equiv M_{k-1}$ and $B_m \equiv M_{k+1}$, as illustrated in Figure~\ref{fig:bezier_dem}. Denote $B^\prime_0 \equiv P_k$ and $B^\prime_j$ as the intersection of $B_0B_j$ and $P_kB_m$. At $B_j$, the turning angle of the UAV flying along the interpolated curve is $\angle B_{j+1}B_jB'_{j}$. 

\begin{theorem}\label{the:2}
The curve $B\left(\tau\right)$ always satisfies the turning angle constraint, i.e.,
\begin{equation}
     \angle B_{j+1}B_jB'_{j} < \gamma_{k} \leq \gamma_\text{max} \quad \forall j \in \left\{0,...,m-1\right\}
\end{equation}
\end{theorem}

\begin{proof}
For all $ j \in \left\{0,...,m-1\right\}$, consider triangle $\triangle B'_{j+1}B'_jB_{j}$, we have
\begin{equation}
   \angle{B'_{j+1}B_jB'_{j}} + \angle{B_jB'_{j+1}B'_{j}} + \angle{B'_{j+1}B'_{j}B_j} = \pi.
\end{equation}
Hence,
\begin{equation}
    \angle{B_{j+1}B_jB'_{j}}= \angle{B'_{j+1}B_jB'_{j}}<\pi-\angle{B'_{j+1}B'_{j}B_j}.
    \label{eqn:UAVbezier1}
\end{equation}
On the other hand, consider triangle $\triangle B'_jB'_{j=1}B_{j-1}$, we have
\begin{equation}
    \angle{B'_{j+1}B'_jB_j} = \angle{B'_{j}B'_{j-1}B_{j-1}} + \angle{B'_{j}B_{j-1}B'_{j-1}}.
\end{equation}
and hence 
\begin{equation}
    \angle{B'_{j+1}B'_jB_j} > \angle{B'_{j}B'_{j-1}B_{j-1}}.
\end{equation}
As a result, we get
\begin{equation}  
     \pi- \angle{B'_{j+1}B'_jB_{j}} < \pi- \angle{B'_{j}B'_{j-1}B_{j-1}}.
     \label{eqn:UAVbezier2}
\end{equation}
Applying \eqref{eqn:UAVbezier2} for $\forall j \in \{1,...,m-1\}$ gives
\begin{equation} \label{eqn:UAVbezier3}
\begin{aligned}
    \pi- \angle{B'_{m}B'_{m-1}B_{m-1}} &<...<\pi- \angle{B'_{j+1}B'_jB_{j}} \\
    &<...< \pi- \angle{B'_{1}B'_{0}B_{0}} = \gamma_k.
\end{aligned}
\end{equation}
From \eqref{eqn:UAVbezier1} and \eqref{eqn:UAVbezier3}, we obtain
\begin{equation}
    \angle{B_{j+1}B_jB'_{j}} < \gamma_k.
\end{equation}
\end{proof}

\begin{remark}
With Theorem~\ref{the:1} and Theorem~\ref{the:2}, we prove that our path smoothing using second-order Bezier curves not only satisfies the UAV's dynamic constraint but also ensures sufficient distance from obstacles for its safe operation. 
\end{remark}

\begin{remark}
When operating at a certain forward velocity $v_{x}$, the energy consumption of the UAV is proportional to its travel distance. Since MultiRRT minimizes the path length, it reduces energy consumption.
\end{remark}

\section{Results and Discussion}\label{sec:results}
To evaluate the performance of the proposed algorithm, we have conducted a number of comparisons and experiments with details as follows.

\subsection{Scenario setup}
    The UAVs used in simulations and experiments are quad-rotors named 3DR Solo\footnote{Documentation of the 3DR Solo drone is available \href{https://www.drones.nl/media/files/drones/1456527966-3dr-solo-v8-02-05-16.pdf}{\tt{here}}.}, as shown in Figure~\ref{fig:exper}. Their parameters are chosen as in Table~\ref{tbl:params}, which results in the maximum turning angle of $\gamma_\text{max}=75^o$. The tree-expansion distance $\delta$ is chosen as 50, and the maximum number of iterations is set to 5000.
\begin{table}
\centering
\caption{UAV parameters}
\label{tbl:params}
\begin{tabular}{p{1.2cm}p{1.8cm}p{4.5cm}}
\hline \hline
\textbf{Notation}         & \textbf{Value}       & \textbf{Meaning}  \\ 
\hline \hline
$m$         & 1.50       & Mass of the UAV (kg)                                    \\ \hline
$g$         & 9.81       & Gravity acceleration (m/s$^2$)    \\ \hline     $v_x$        & 8.0      & Velocity in the x direction  (m/s)               \\ \hline                   
$\Omega_\text{max}$       & 1000       & Angular velocities of the propellers                                   \\ \hline
$c_t$       & $2.9\times10^{-5}$     & Thrust coefficient                                     \\ \hline
$c_f$       & $1.1\times10^{-6}$       & Friction coefficient                                    \\ \hline \hline
\end{tabular}
\end{table}

\begin{figure*}
    \centering
    \begin{subfigure}[b]{0.245\textwidth}
    \includegraphics[width=\textwidth]{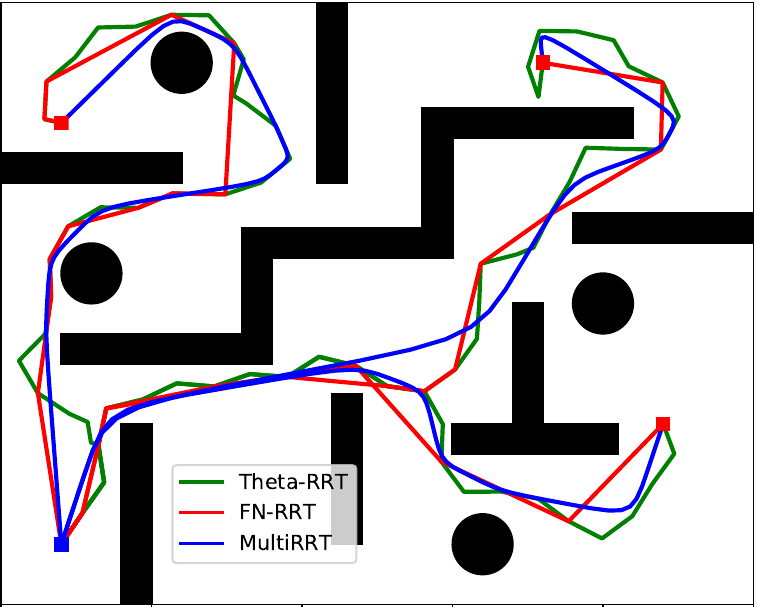}
    \caption{Scenario 1}
    \end{subfigure}
    \begin{subfigure}[b]{0.245\textwidth}
    \includegraphics[width=\textwidth]{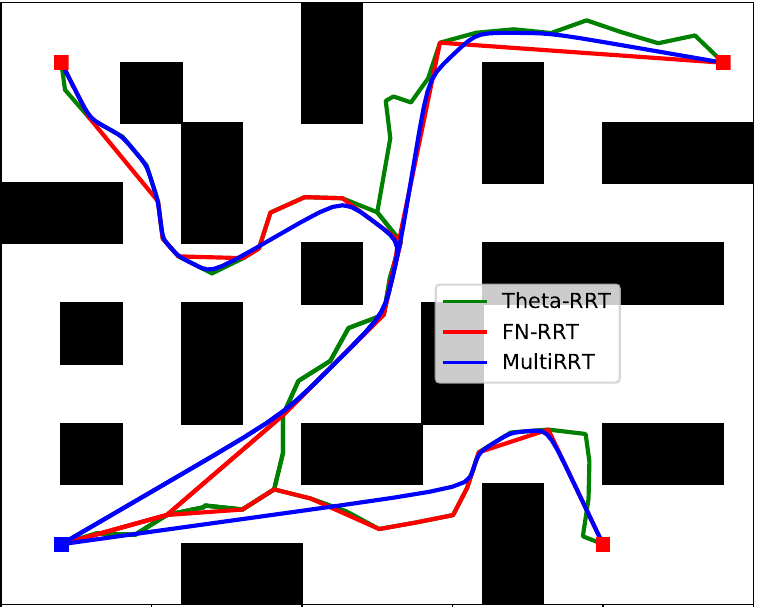}
    \caption{Scenario 2}
    \end{subfigure}
    \begin{subfigure}[b]{0.245\textwidth}
    \includegraphics[width=\textwidth]{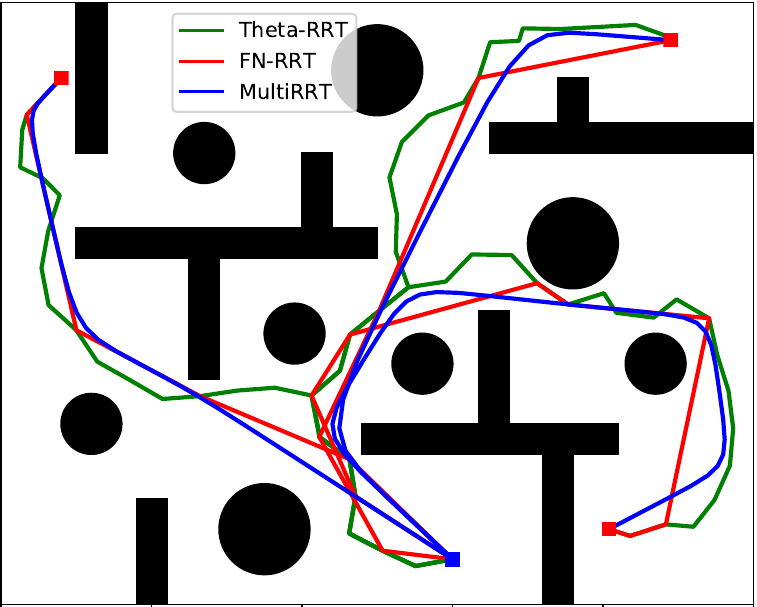}
    \caption{Scenario 3}
    \end{subfigure}
    \begin{subfigure}[b]{0.245\textwidth}
    \includegraphics[width=\textwidth]{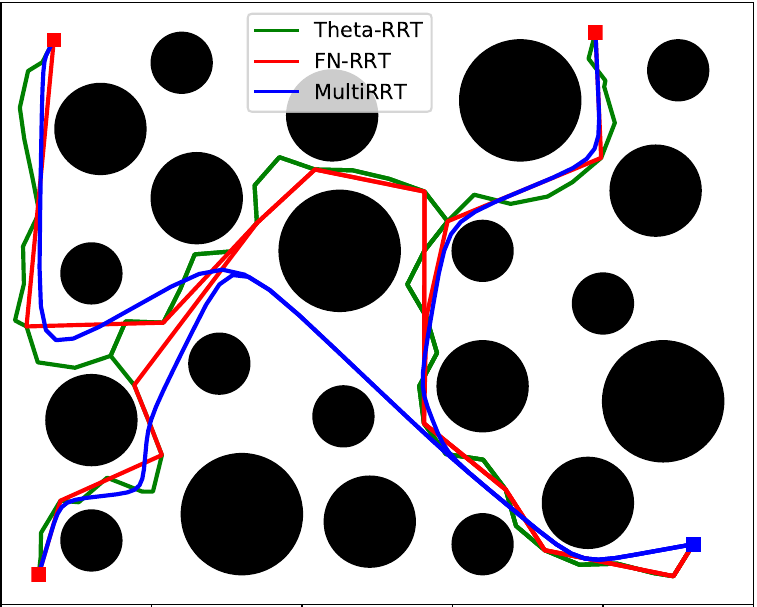}
    \caption{Scenario 4}
    \end{subfigure}
    \caption{The paths generated by the RRT algorithms in four scenarios}
    \label{fig:reduce_path}
\end{figure*}

In evaluation, four scenarios are generated with different complexity levels, as shown in Figure~\ref{fig:reduce_path}. Three metrics are used to evaluate the algorithms with definition as follows~\cite{Phung2021,PuenteCastro2022,9078055}.
\begin{enumerate}[label=\roman*.]
\item \textit{Path length:} The path length is defined as the average length of all generated paths, i.e., 
\begin{equation}
    F_L = \dfrac{1}{n}\sum_{i=1}^n\sum_{j=1}^{N_i}{\left\Vert P_{ij}P_{i,j+1}\right\Vert} .
\end{equation}

\item \textit{Smooth score:} The smoothness score, $F_S$, is determined based on the angles between consecutive flight path segments, referred to as the turning angles $\gamma$, and is calculated as follows:

\begin{equation}
    F_S=\dfrac{1}{n(N_i-1)}\sum_{i=1}^n\sum_{j=2}^{N_i}\gamma_{ij}.
\end{equation}
The smaller the angle values, the smoother the flight paths, and vice versa. Thus, a small smooth score implies a good path as less turning operation is needed.

\item \textit{Computation time:} Computation time, $F_T$, is the time required for the algorithm to generate paths to all the targets.
\end{enumerate}

In our comparison, each algorithm is executed 10 times for each scenario. The results are then presented by box plots that include the minimum, maximum, average, and standard deviation of the obtained data.

\subsection{Node reduction evaluation}

\begin{figure}
\centering
    \centering
    \includegraphics[width=0.49\textwidth]{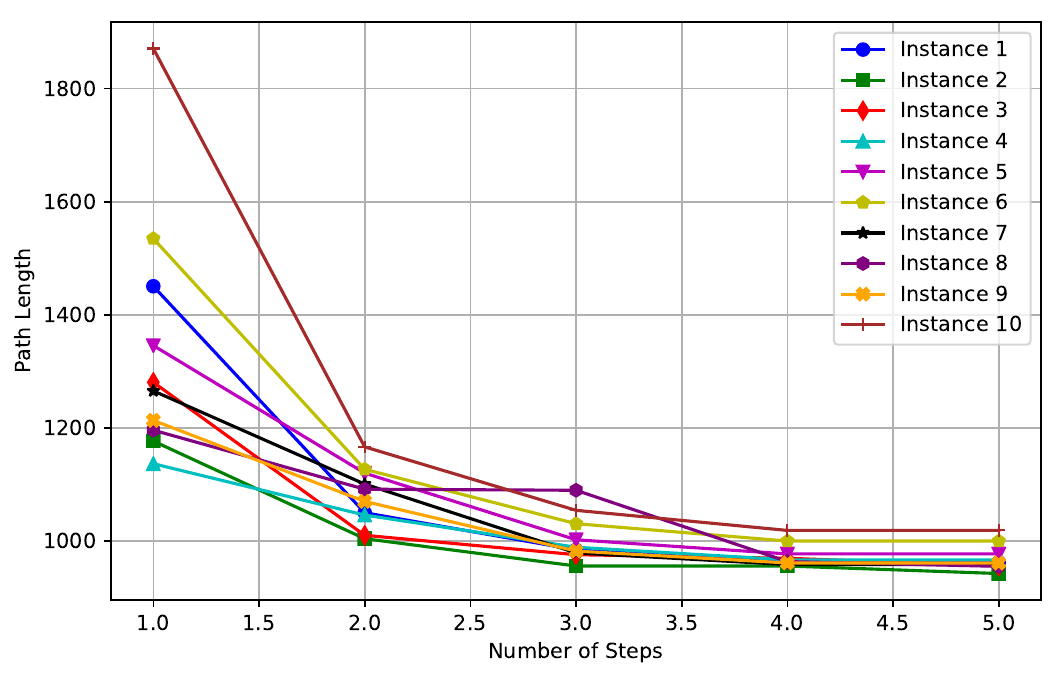}
    \caption{Reduction in length of ten path instances during the node reduction process}
    \label{fig:reduce_length_step}
\end{figure}

\begin{figure*}
    \centering
    \begin{subfigure}[b]{0.22\textwidth}
    \includegraphics[width=\textwidth]{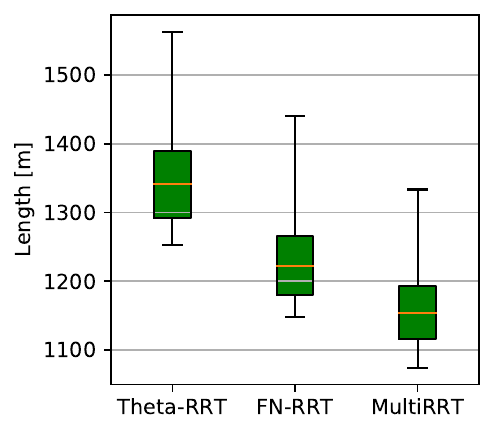}
    \caption{Scenario 1}
    \end{subfigure}
    \begin{subfigure}[b]{0.22\textwidth}
    \includegraphics[width=\textwidth]{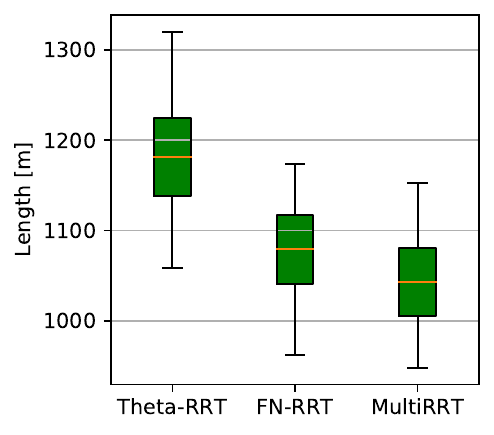}
    \caption{Scenario 2}
    \end{subfigure}
    \begin{subfigure}[b]{0.22\textwidth}
    \includegraphics[width=\textwidth]{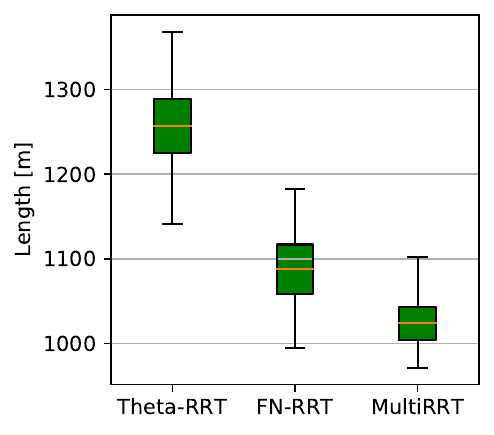}
    \caption{Scenario 3}
    \end{subfigure}
    \begin{subfigure}[b]{0.22\textwidth}
    \includegraphics[width=\textwidth]{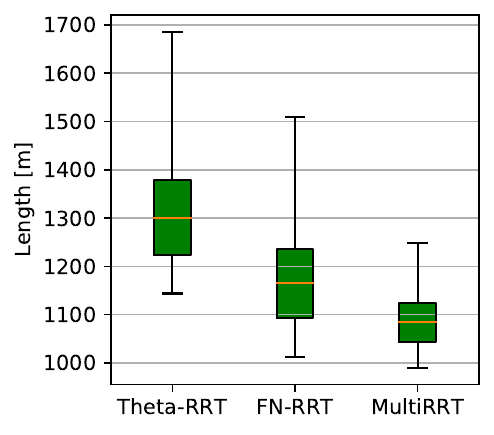}
    \caption{Scenario 4}
    \end{subfigure}
    \caption{The path length $F_L$ of the comparing algorithms}
    \label{fig:reduce_length}
\end{figure*}

\begin{figure*}
    \centering
    \begin{subfigure}[b]{0.22\textwidth}
    \includegraphics[width=\textwidth]{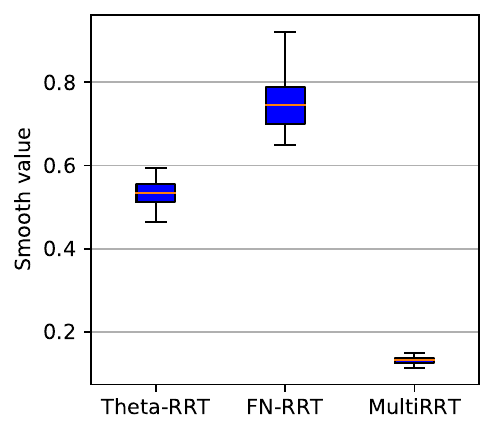}
    \caption{Scenario 1}
    \end{subfigure}
    \begin{subfigure}[b]{0.22\textwidth}
    \includegraphics[width=\textwidth]{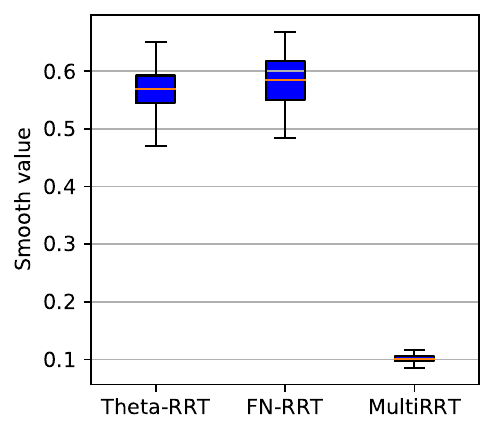}
    \caption{Scenario 2}
    \end{subfigure}
    \begin{subfigure}[b]{0.22\textwidth}
    \includegraphics[width=\textwidth]{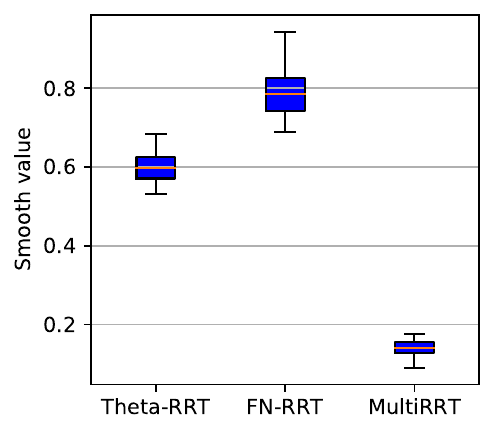}
    \caption{Scenario 3}
    \end{subfigure}
    \begin{subfigure}[b]{0.22\textwidth}
    \includegraphics[width=\textwidth]{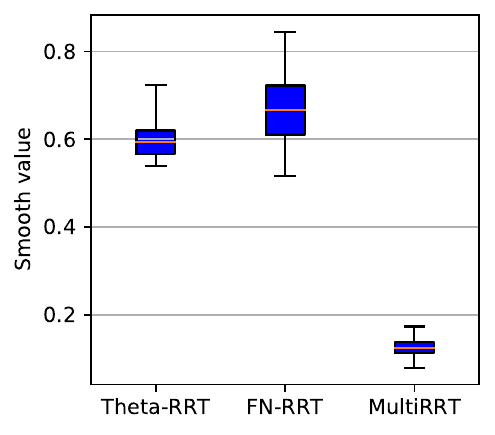}
    \caption{Scenario 4}
    \end{subfigure}
    \caption{The smooth score $F_S$ of the comparing algorithms}
    \label{fig:reduce_smooth}
\end{figure*}

\begin{figure*}
    \centering
    \begin{subfigure}[b]{0.22\textwidth}
    \includegraphics[width=\textwidth]{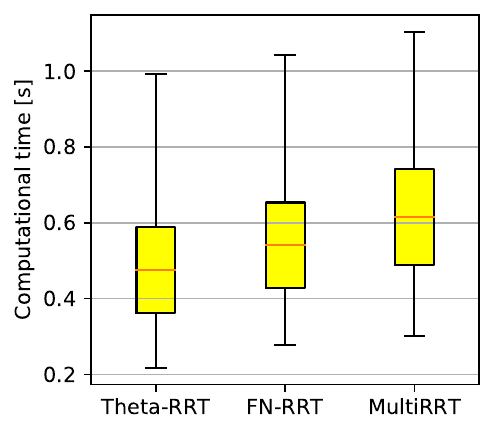}
    \caption{Scenario 1}
    \end{subfigure}
    \begin{subfigure}[b]{0.22\textwidth}
    \includegraphics[width=\textwidth]{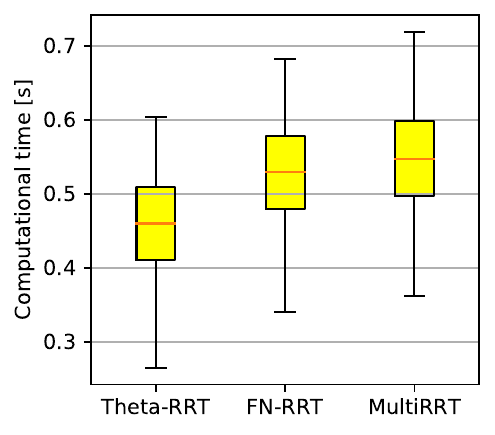}
    \caption{Scenario 2}
    \end{subfigure}
    \begin{subfigure}[b]{0.22\textwidth}
    \includegraphics[width=\textwidth]{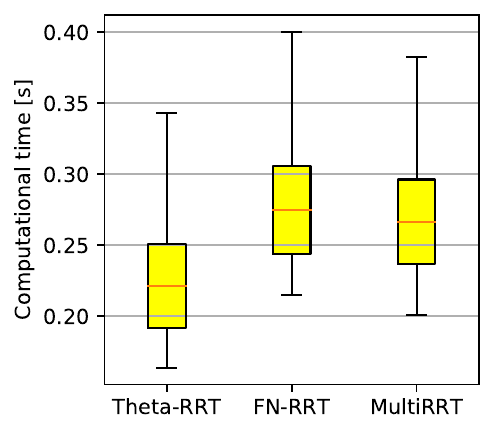}
    \caption{Scenario 3}
    \end{subfigure}
    \begin{subfigure}[b]{0.22\textwidth}
    \includegraphics[width=\textwidth]{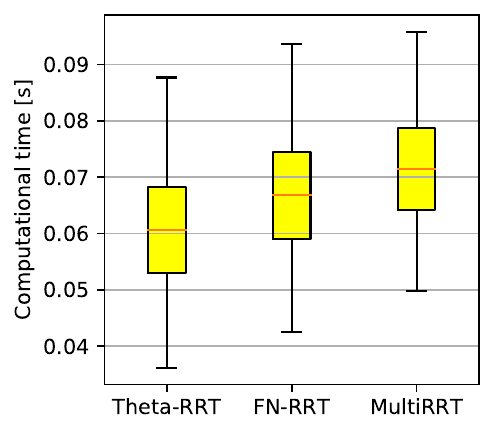}
    \caption{Scenario 4}
    \end{subfigure}
    \caption{The execution time $F_T$ of the comparing algorithms}
    \label{fig:reduce_time}
\end{figure*}

\begin{figure*}[!]
    \centering
    \begin{subfigure}[b]{0.245\textwidth}
    \includegraphics[width=\textwidth]{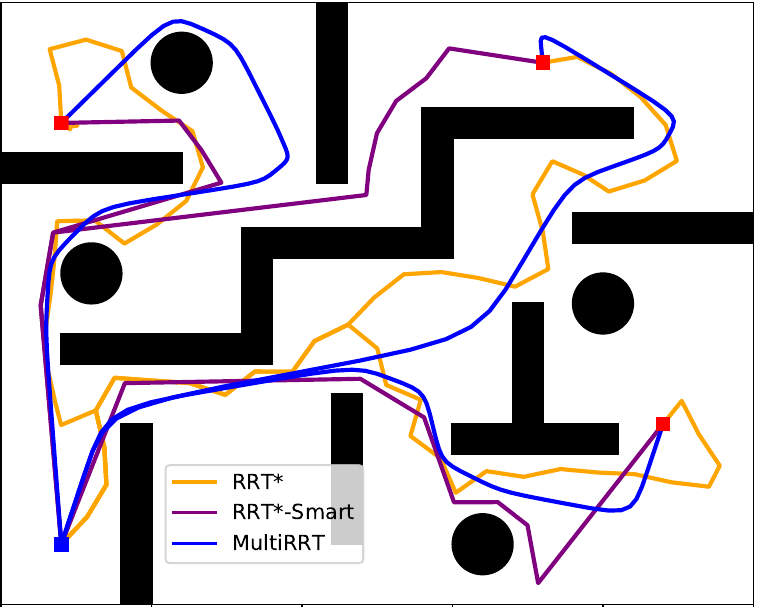}
    \caption{Scenario 1}
    \end{subfigure}
    \begin{subfigure}[b]{0.245\textwidth}
    \includegraphics[width=\textwidth]{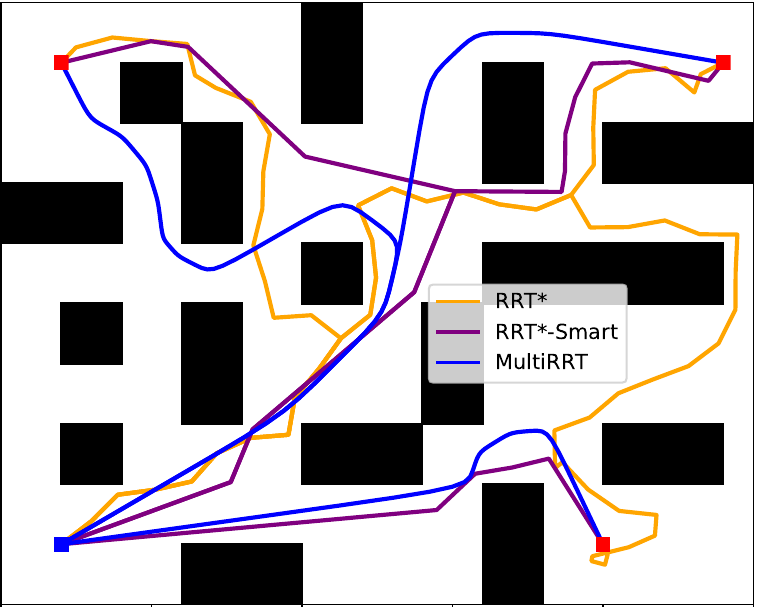}
    \caption{Scenario 2}
    \end{subfigure}
    \begin{subfigure}[b]{0.245\textwidth}
    \includegraphics[width=\textwidth]{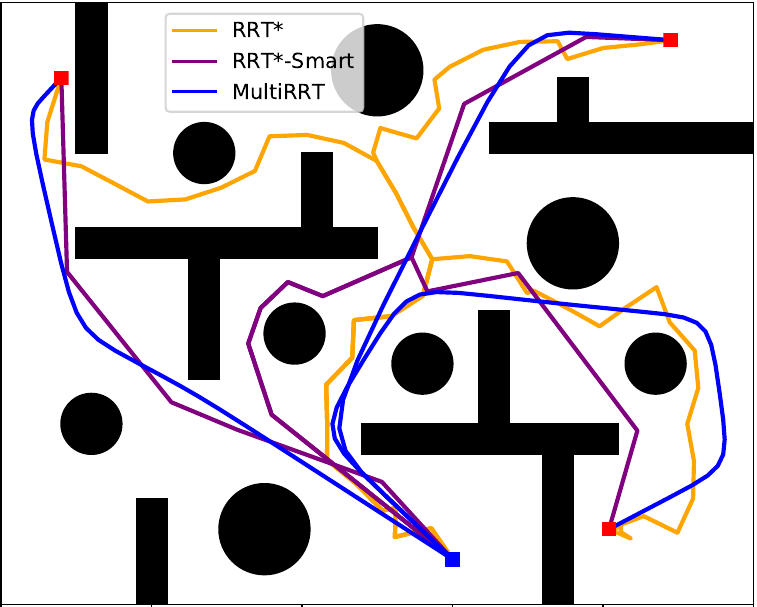}
    \caption{Scenario 3}
    \end{subfigure}
    \begin{subfigure}[b]{0.245\textwidth}
    \includegraphics[width=\textwidth]{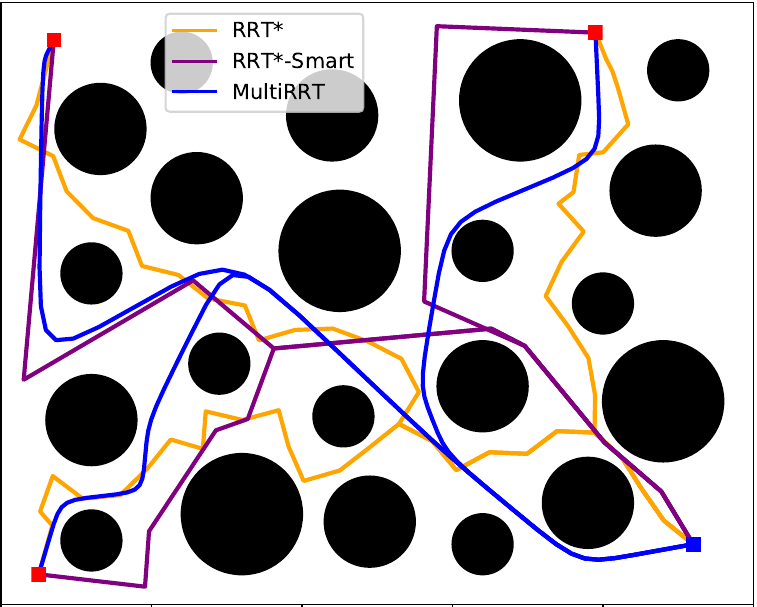}
    \caption{Scenario 4}
    \end{subfigure}
    \caption{The paths generated by the RRT algorithms}
    \label{fig:rrt_path}
\end{figure*}

\begin{figure*}[!h]
    \centering
    \begin{subfigure}[b]{0.22\textwidth}
    \includegraphics[width=\textwidth]{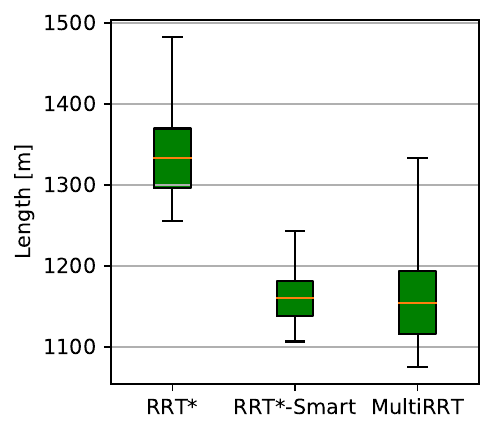}
    \caption{Scenario 1}
    \end{subfigure}
    \begin{subfigure}[b]{0.22\textwidth}
    \includegraphics[width=\textwidth]{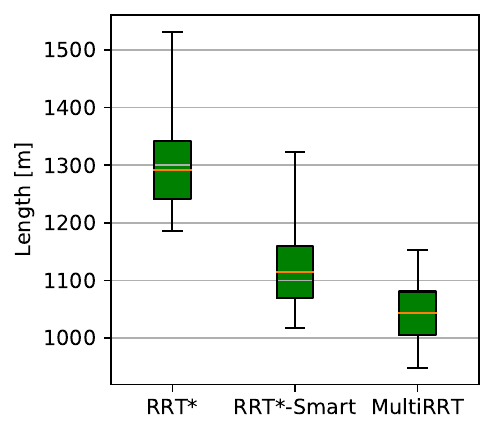}
    \caption{Scenario 2}
    \end{subfigure}
    \begin{subfigure}[b]{0.22\textwidth}
    \includegraphics[width=\textwidth]{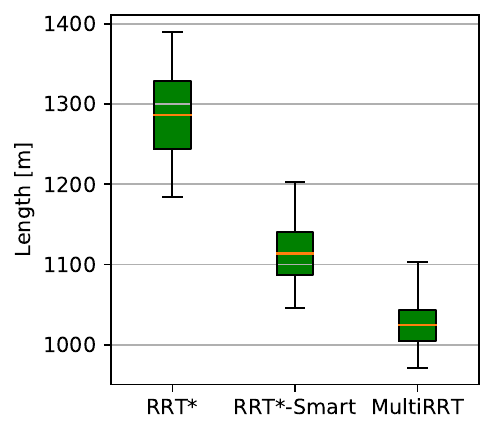}
    \caption{Scenario 3}
    \end{subfigure}
    \begin{subfigure}[b]{0.22\textwidth}
    \includegraphics[width=\textwidth]{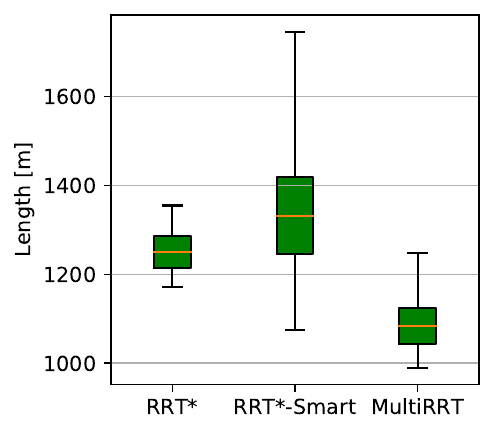}
    \caption{Scenario 4}
    \end{subfigure}
    \caption{The path length $F_L$ of the RRT algorithms}
    \label{fig:rrt_length}
\end{figure*}

\begin{figure*}
    \centering
    \begin{subfigure}[b]{0.22\textwidth}
    \includegraphics[width=\textwidth]{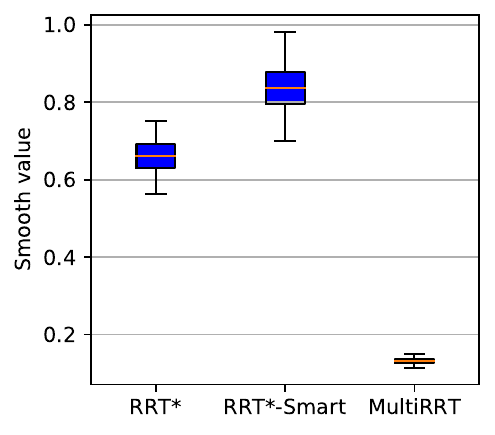}
    \caption{Scenario 1}
    \end{subfigure}
    \begin{subfigure}[b]{0.22\textwidth}
    \includegraphics[width=\textwidth]{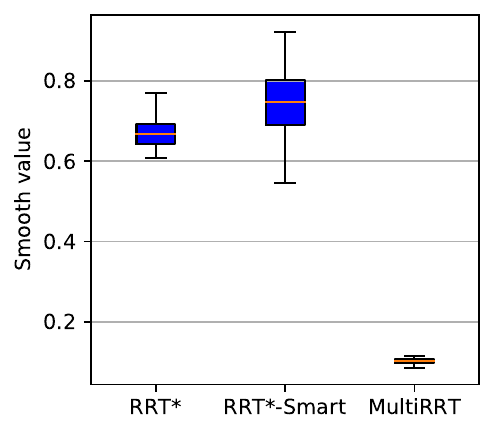}
    \caption{Scenario 2}
    \end{subfigure}
    \begin{subfigure}[b]{0.22\textwidth}
    \includegraphics[width=\textwidth]{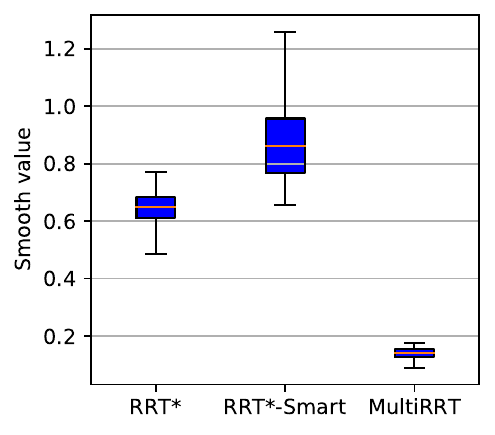}
    \caption{Scenario 3}
    \end{subfigure}
    \begin{subfigure}[b]{0.22\textwidth}
    \includegraphics[width=\textwidth]{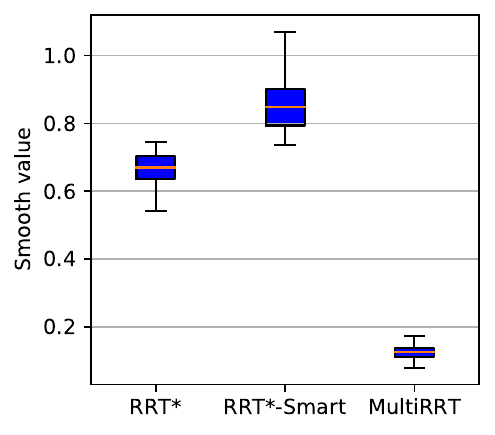}
    \caption{Scenario 4}
    \end{subfigure}
    \caption{The smooth score $F_S$ of the RRT algorithms}
    \label{fig:rrt_smooth}
\end{figure*}

\begin{figure*}
    \centering
    \begin{subfigure}[b]{0.22\textwidth}
    \includegraphics[width=\textwidth]{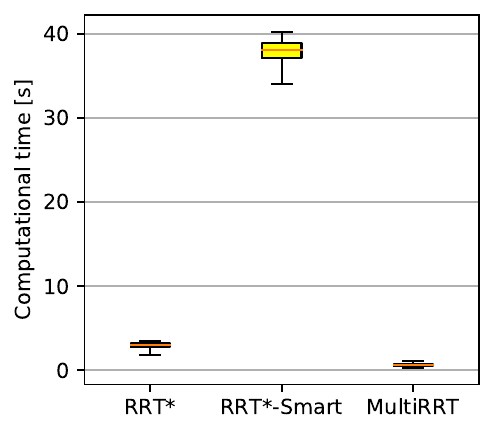}
    \caption{Scenario 1}
    \end{subfigure}
    \begin{subfigure}[b]{0.22\textwidth}
    \includegraphics[width=\textwidth]{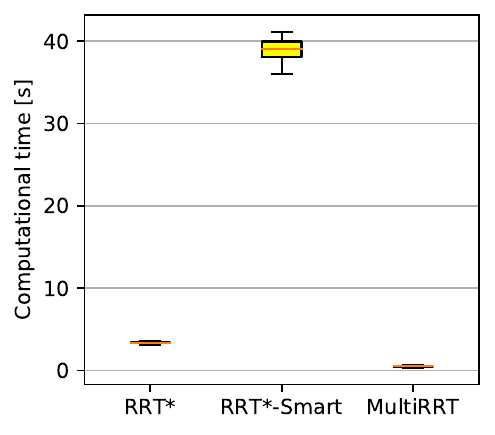}
    \caption{Scenario 2}
    \end{subfigure}
    \begin{subfigure}[b]{0.22\textwidth}
    \includegraphics[width=\textwidth]{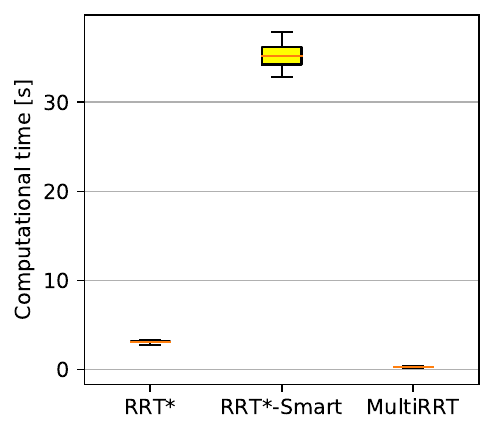}
    \caption{Scenario 3}
    \end{subfigure}
    \begin{subfigure}[b]{0.22\textwidth}
    \includegraphics[width=\textwidth]{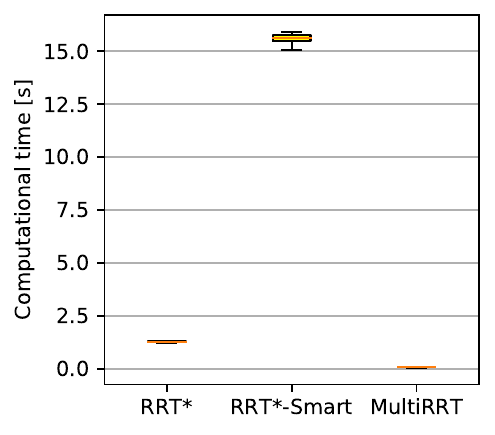}
    \caption{Scenario 4}
    \end{subfigure}
    \caption{The execution time of the RRT algorithms}
    \label{fig:rrt_time}
\end{figure*}

In this section, we evaluate the effectiveness of our node reduction mechanism by comparing it with two popular RRT variants, Theta-RRT~\cite{7487439} and FN-RTT~\cite{Qi2023}. The Theta-RRT does not use node reduction but geometric information to optimize the path. The FN-RRT, on the other hand, employs an intensive node reduction mechanism where intermediate nodes are removed until a potential collision is detected.

The paths generated in the four scenarios are shown in Figure~\ref{fig:reduce_path}. It can be seen that all methods can create collision-free paths from the start to the target locations. However, the path length $F_L$ of the proposed method is shortest among all scenarios. This result can be further confirmed in Figure~\ref{fig:reduce_length}, which shows the mean and variance of the path lengths in all scenarios. Thanks to our node reduction mechanism, MultiRRT obtains the smallest mean values for all four scenarios. Its variance is also smaller than other algorithms, which implies a more stable convergence. Figure {\ref{fig:reduce_length_step}} illustrates the convergence process in which the path length reduces in all ten trials as redundant nodes are progressively removed at each step.

Figure~\ref{fig:reduce_smooth} shows the smooth score of the algorithms. MultiRRT obtains significantly smaller values in both mean and variance than the other algorithms in all scenarios. The rationale for this is the integration of dynamic constraints and the Bezier smoothing mechanism in MultiRRT. However, they also lead to extra computation time, which is about 0.1 s compared to the fastest method, as depicted in Figure {\ref{fig:reduce_time}}. This small extra calculation time does not affect the capability of MultiRRT as its overall execution time (less than 1 s) is still much faster than other approaches (see Table {\ref{tbl:other_al})} and sufficient for online replanning.

\subsection{Comparison with other RRT variants}

In this section, comparisons with other suboptimal variants of RRT, including RRT*~\cite{Karaman2011,Wang2024} and RRT*-Smart~\cite{Nasir2013}, are conducted to further evaluate the performance of the proposed method. As shown in Figure~\ref{fig:rrt_path}, all paths generated by the RTT variants are able to reach the targets without collision. Their path length, smoothness, and computation time in all scenarios are shown in Figures~\ref{fig:rrt_length}, \ref{fig:rrt_smooth}, and \ref{fig:rrt_time}, respectively. It can be seen that MultiRRT achieves the best score in all metrics. Especially the smooth score of the proposed method is at least three times better than the other algorithms due to the Bezier interpolation. This is an important feature of our algorithm to generate paths that the UAVs can accurately follow in practice.

\subsection{Comparison with other approaches}
\begin{table*}
\caption{Comparison with other approaches}
\label{tbl:other_al}
\centering
\begin{tabular}{{C{1.3cm}|C{1.1cm}C{1.1cm} C{1.1cm}|C{1.1cm}C{1.1cm} C{1.1cm}|C{1.1cm}C{1.1cm} C{1.1cm}}}
\hline \hline
        \multirow{2}{*}{\begin{tabular}[c]{@{}c@{}} Density\\(obs/m$^2$)\end{tabular}} & \multicolumn{3}{c|}{$F_L$ (m)} & \multicolumn{3}{c|}{$F_S$ } & \multicolumn{3}{c}{$F_T$ (s) } \\ 
        \cline{2-10}
         & PSO & Voronoi & MultiRRT & PSO & Voronoi & MultiRRT & PSO & Voronoi & MultiRRT \\ 
        \hline
        0.16 & \textbf{924.843} & 1355.831 & 937.671 & 0.772 & 0.480 & \textbf{0.097} & 106.282  & 55.778 & \textbf{0.073}\\ 
        0.22 & 1182.039 & 1255.206 & \textbf{1004.041} & 0.697 &  0.288 & \textbf{0.124} &174.213  & 91.723 & \textbf{0.047} \\ 
        0.3  & 1213.339 & 1698.931 & \textbf{1018.017} & 0.976 & 0.199 & \textbf{0.117}  & 218.718 & 129.561  & \textbf{0.056}\\ 
        \hline \hline
        \multicolumn{4}{l}{*Bold values indicate the best scores}
\end{tabular}
\end{table*}

To validate the performance of the proposed algorithm, we compared it with other approaches including a Voronoi-based method\mbox{\cite{9779232}} and a particle swarm optimization (PSO)-based method\mbox{\cite{Das2020}}. The comparisons were conducted in environments with varying obstacle densities (obs/m$^2$) of 0.16, 0.22, and 0.3. The results presented in Table~\mbox{\ref{tbl:other_al}} show that while PSO achieves the shortest path at a low obstacle density (0.16 obs/m$^2$), its performance degrades in more complex scenarios. The Voronoi approach is ineffective with the worst path length scores in all scenarios. In contrast, the proposed MultiRRT algorithm shows superior performance and stability across different scenarios with nearly constant metric values even as obstacle density increases. Notably, its computation time is much faster than the other methods, making it suitable for real-time execution, which is essential for cooperative path planning.

\begin{table}
\caption{Scalability analysis of MultiRRT}
\label{tbl:scalability}
\centering
\begin{tabular}{C{1.5cm}C{0.95cm}C{0.95cm}C{0.95cm}C{0.95cm}C{0.95cm}}
\hline \hline
No. of UAVs & 3        & 5       & 6       & 8       & 10        \\ \hline
$F_L$ (m)      & 1138.815 & 1162.778 & 1212.089 & 1154.804 & 1137.892 \\
$F_S$      & 0.140    & 0.142   & 0.136   & 0.131   & 0.122     \\
$F_T$ (s)      & 0.685    & 0.745   & 0.706   & 0.911   & 0.930    \\ \hline \hline
\end{tabular}
\end{table}

\begin{figure*}
\centering
    \begin{subfigure}[b]{0.485\textwidth}
    \centering
    \includegraphics[width=\textwidth]{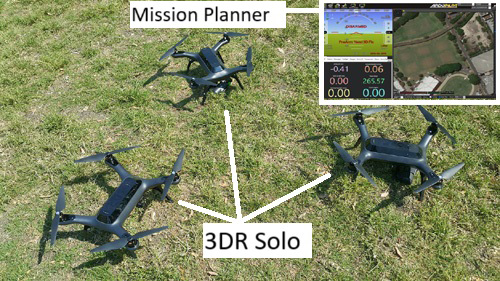}
    \caption{Three 3DR Solo drones with Mission Planner}
    \label{fig:drone1}
    \end{subfigure}
    \begin{subfigure}[b]{0.4\textwidth}
    \centering
    \includegraphics[width=\textwidth]{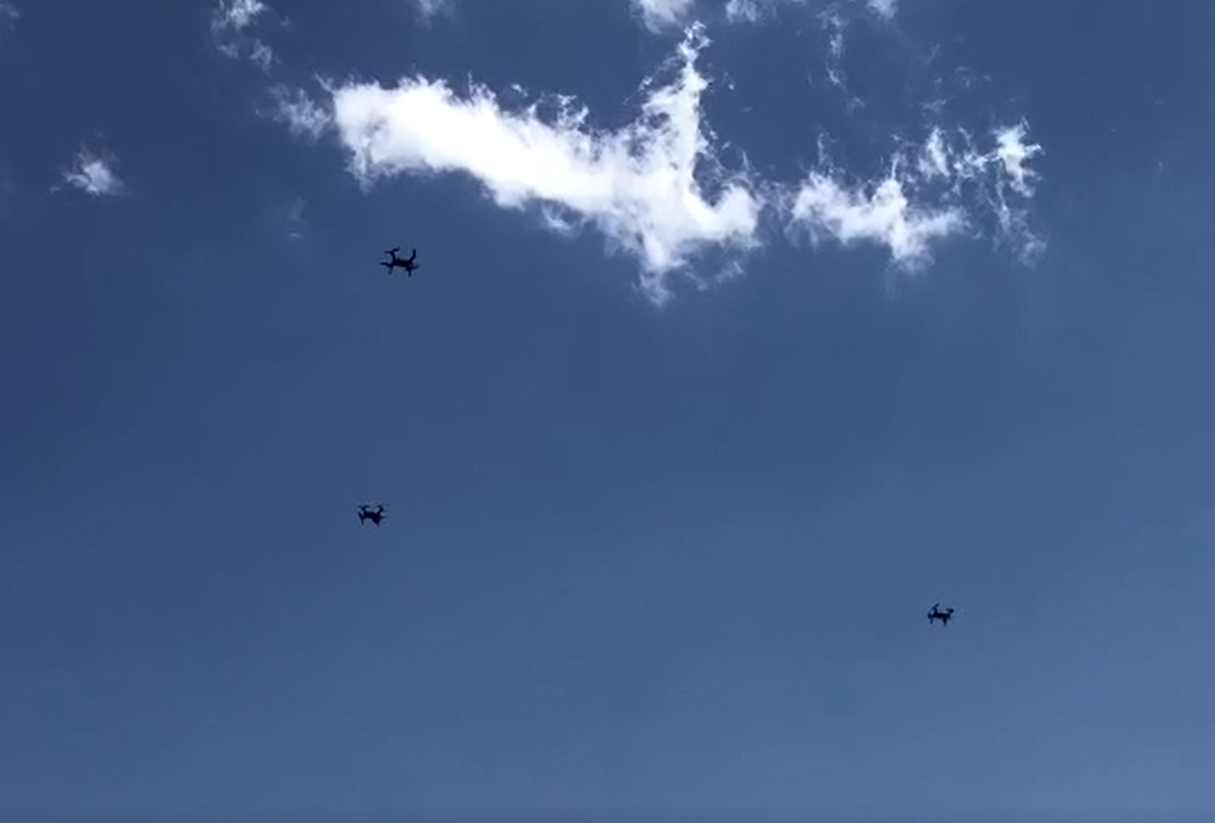}
    \caption{Three 3DR Solo drones in operation}
    \label{fig:drone2}
    \end{subfigure}
    \caption{The drones used for experiments}
    \label{fig:exper}
\end{figure*}

\begin{figure}
\centering
    \centering
    \includegraphics[width=0.47\textwidth]{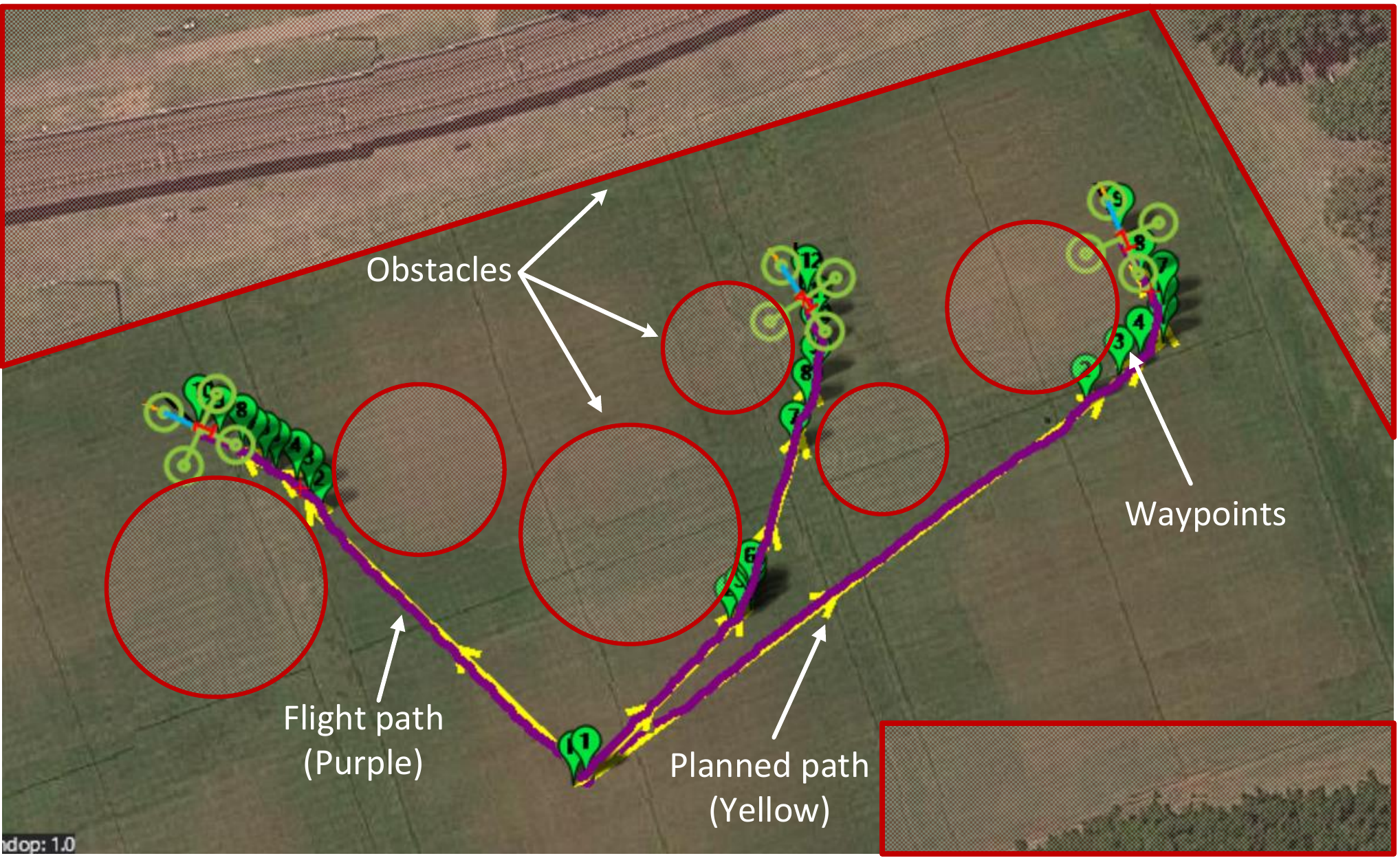}
    \caption{Experimental result with the planned (yellow) and actual (purple) flight paths of three UAVs programmed via Mission Planner. Note that the drone symbol of Mission Planner does not reflect the actual size of the UAVs.}
    \label{fig:exp_path}
\end{figure}

\subsection{Scalability analysis}

To analyze the scalability of the proposed method, we evaluated its performance metrics across different numbers of UAVs. As shown in Table~\mbox{\ref{tbl:scalability}}, the path length and smoothness scores remain stable. The extra execution time also remains minimal due to the multi-goal expansion within a single random tree in MultiRRT. These results confirm that the proposed algorithm can be scaled to accommodate a large number of UAVs.

\subsection{Experimental validation}
To validate the applicability of MultiRRT in generating flyable paths, experiments have been conducted in an area of size 180 m $\times$ 120 m within a park. The UAVs used are 3DR Solo drones that can be programmed to fly autonomously via a ground control station called Mission Planner. Figure~\ref{fig:drone1} and Figure~\ref{fig:drone2} respectively show the three UAVs with Mission Planner and their operation in experiments. The goal is to plan paths for the UAVs from the same starting location to three different goal locations and then upload the paths to the UAVs for automatic flight. To simulate practical scenarios, the environment is augmented with obstacles represented by red line areas shown in Figure~\ref{fig:exp_path}. These areas represent both physical objects such as trees and bridges, and non-physical objects like no-fly zones. The start location is set at the longitude and latitude of $(-33.876289,151.19243)$ and the goal locations are set at $(-33.875736,151.192739)$, $(-33.875898,151.191895)$, and $(-33.875669,151.193161)$. The flight altitudes are set to 5 m, 10 m, and 15 m to avoid collisions among the UAVs. This information is used as input to our MultiRRT path planning program written in Python to generate paths and reference velocities for the UAVs. Each path is represented by a set of waypoints, which is then transmitted to its corresponding UAV via Mission Planner to fly automatically. During operation, if major adjustments are required such as modifying the path or reference velocity, MultiRRT can be re-executed to replan the UAVs due to its fast execution time.

Figure~\ref{fig:exp_path} shows the planned (yellow) and actual (purple) flight paths of the three UAVs. It can be seen that the algorithm can generate collision-free paths for all the UAVs to reach their target locations. The overlapping between the actual flight paths and the planned paths means that the generated paths meet the safety and dynamics constraints for the UAVs to follow. Especially, the UAVs show good tracking performance at curved segments. The results thus confirm the validity of the proposed algorithm for practical flights of multiple UAVs.
\section{Conclusion}\label{sec:conclusion}
In this paper, we have presented a novel algorithm named MultiRRT for cooperative path planning of multiple UAVs. The algorithm extends the branch expanding mechanism of RRT to reach multiple goal positions. Importantly, it integrates new features related to node reduction, dynamics constraint and Bezier interpolation to ensure the optimality, safety and feasibility of the paths generated. Unlike existing Bezier fitting techniques, our approach considers obstacle locations and dynamics limits of the UAVs to calculate relevant interpolated curves. For the first time, we prove that the interpolation meets the dynamics and safety constraints to generate feasible flight paths. Simulation, comparison, and experiment results confirm the validity and effectiveness of our approach for multiple UAV operations. Our future works will focus on integrating ad hoc communication into the system and expanding the approach to tackle real-time UAV formation for complex collaborative tasks in dynamic environments.


\bibliographystyle{ieeetr}  
\bibliography{mybibfile}

%

\begin{IEEEbiography}[{\includegraphics[width=1in,height=1.25in]{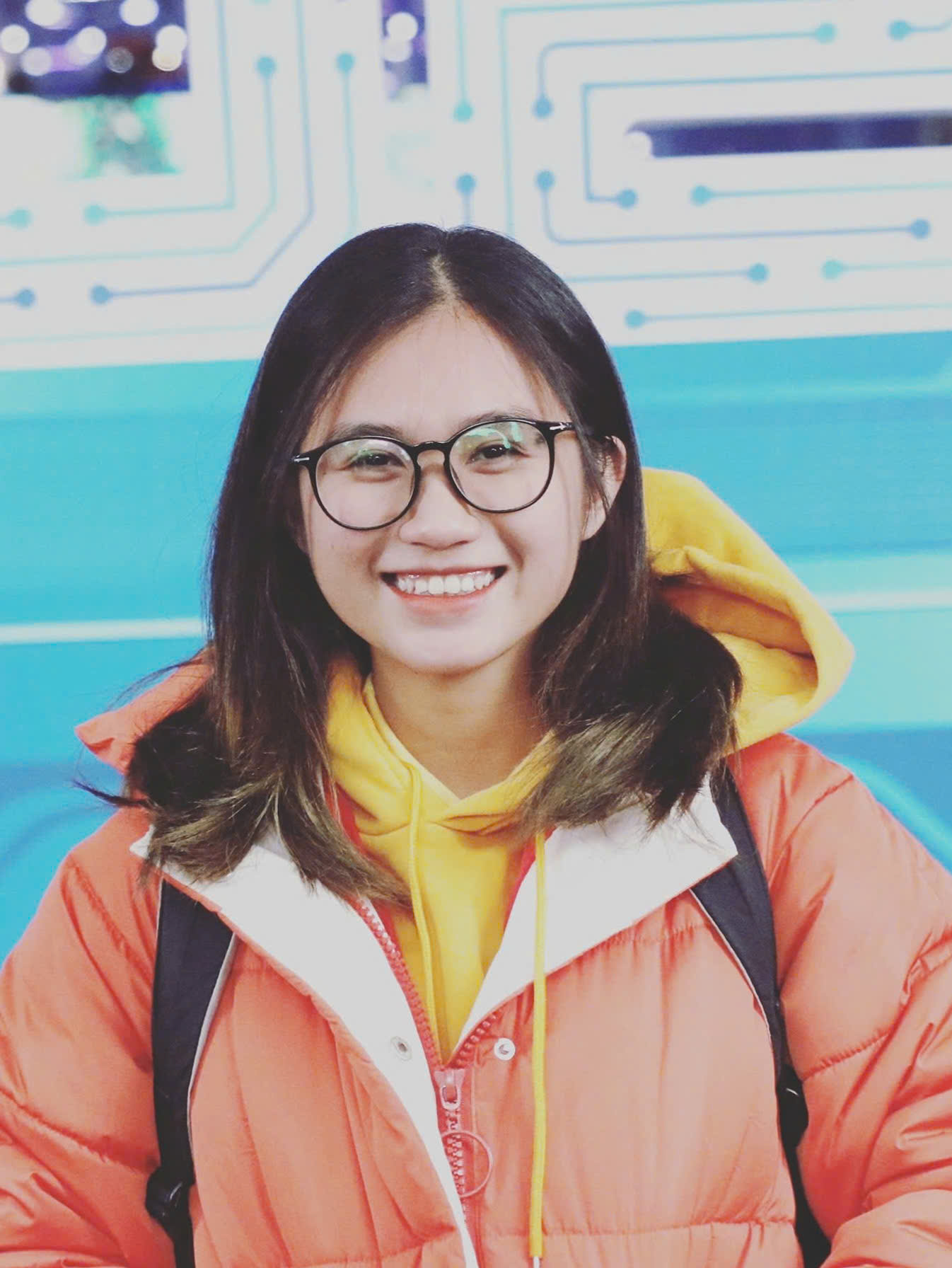}}]{Thu Hang Khuat} received her B.Eng. degree in Robotics Engineering from VNU University of Engineering and Technology in 2023. She is currently pursuing an M.Sc. degree in Electronics Engineering at the same institution. Her research interests include robotics, unmanned aerial vehicles (UAVs), path planning, and swarm optimization.
\end{IEEEbiography}

\begin{IEEEbiography}[{\includegraphics[width=1in,height=1.25in]{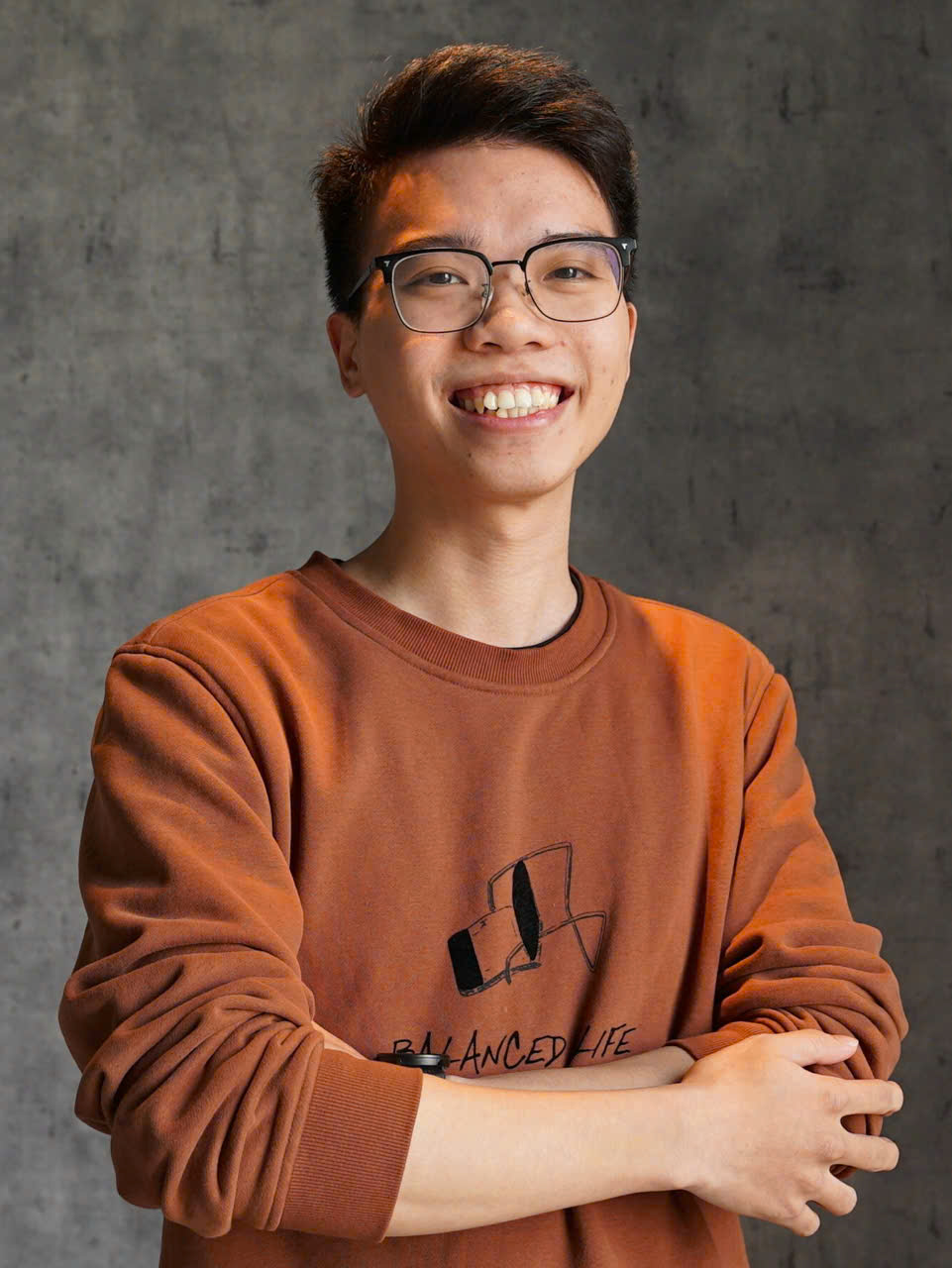}}]{Duy Nam Bui} received his B.Eng. degree in Robotics Engineering and M.Sc. degree in Electronics Engineering from VNU University of Engineering and Technology, Hanoi, Vietnam, in 2022 and 2024, respectively. His current research interests include formation control and cooperative path planning of multi-robot systems, as well as learning and optimization.
\end{IEEEbiography}

\begin{IEEEbiography}[{\includegraphics[width=1in,height=1.25in]{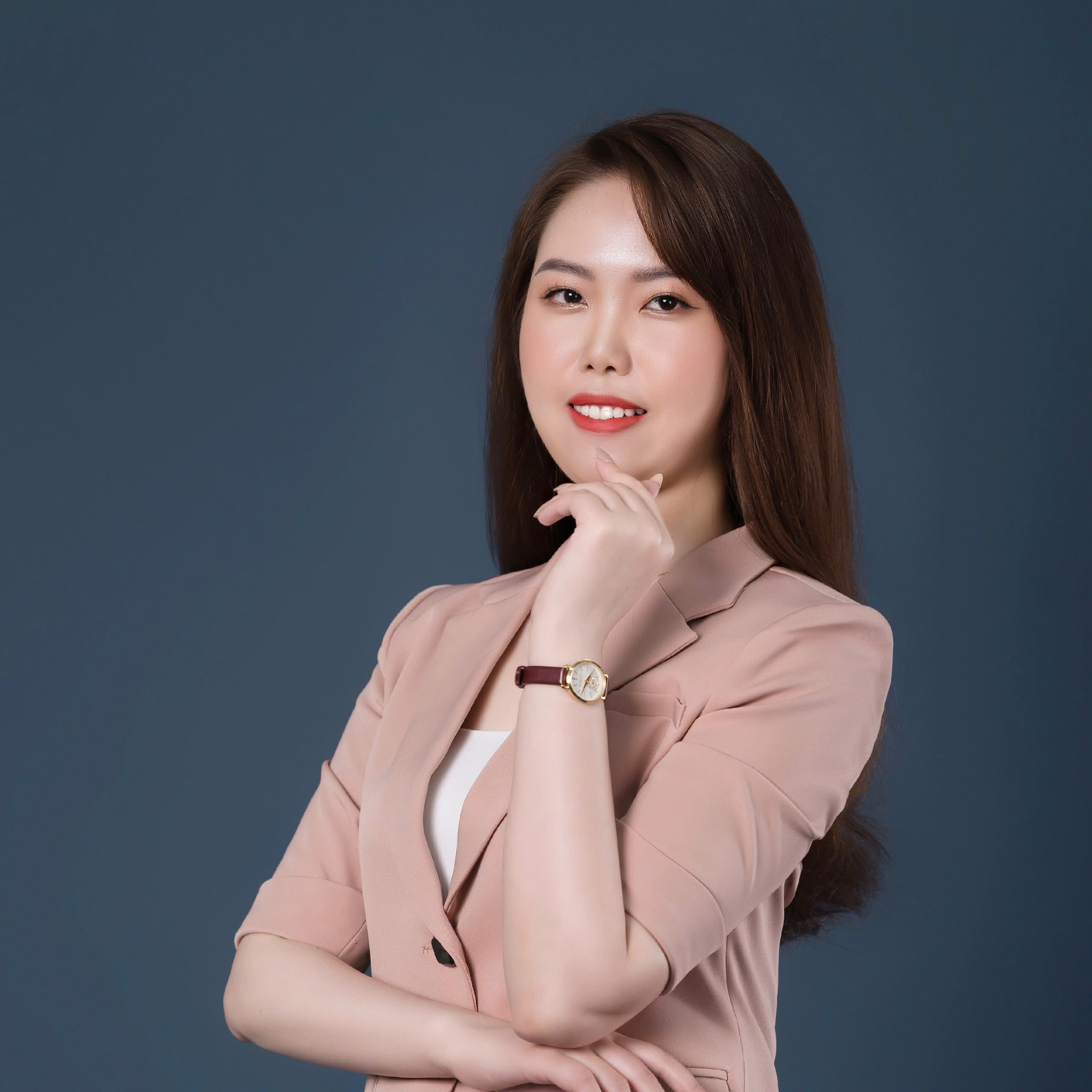}}]{Hoa T. T. Nguyen} received her B.S., M.S. degrees in Control Engineering and Automation from Thai Nguyen University of Technology, Viet Nam in 2013 and 2017, respectively. M.S. Hoa T. Nguyen is currently working as a lecturer and researcher at faculty of electronics engineering, Thai Nguyen University of Technology (TNUT), Vietnam, and also a member of Advanced Wireless Communication Networks (AWCN) Lab. She has interest and expertise in a variety of research topics in Automation, Control Systems, UAV networks, Formation Control, Adaptive Control, and Intelligent Systems.
\end{IEEEbiography}

\begin{IEEEbiography}[{\includegraphics[width=1in,height=1.25in]{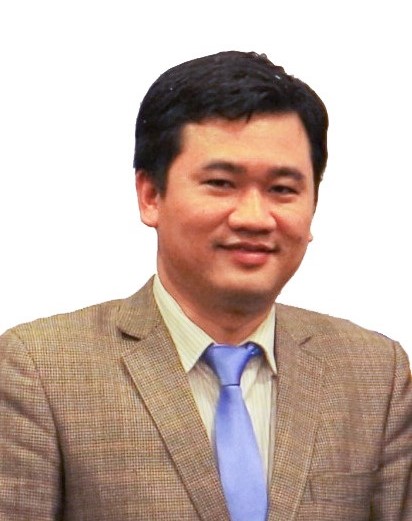}}]{Mien L. Trinh} obtained his Ph.D. from the Russian University of Transport in 2012. He is currently an Associate Professor, Head of Laboratory, Deputy Head of the Department of Cybernetics, and Vice Dean of the Faculty of Electrical and Electronic Engineering at the University of Transport and Communications, Vietnam. His research interests and expertise span a variety of topics in control and automation of machines and production lines, with a particular focus on intelligent control for electric vehicles, electric trains, robotics, and UAV networks.
\end{IEEEbiography}

\begin{IEEEbiography}[{\includegraphics[width=1in,height=1.25in]{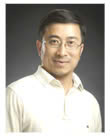}}] {Minh T. Nguyen} received his B.S., M.S., and Ph.D. degrees in Electrical Engineering from Hanoi University of Communication and Transport, Vietnam in 2001; Military Technical Academy, Hanoi, Vietnam in 2007; and Oklahoma State University, Stillwater, OK, USA in 2015, respectively. He is currently the Director of the Human Resource Training and Development Center (HDC) at Thai Nguyen University (TNU), a teaching and research professor at Thai Nguyen University of Technology (TNUT), Vietnam, and the Director of the Advanced Wireless Communication Networks (AWCN) Lab. His research interests and expertise span a wide range of topics in communications, networking, and signal processing, with a particular focus on compressive sensing, wireless/mobile sensor networks, robotics, and UAV networks. He serves as a technical reviewer for several prestigious journals and international conferences. He is also an editor for Wireless Communications and Mobile Computing, EAI Endorsed Transactions on Industrial Networks and Intelligent Systems, and serves as Editor-in-Chief of ICSES Transactions on Computer Networks and Communications.
\end{IEEEbiography}

\begin{IEEEbiography}[{\includegraphics[width=1in,height=1.25in,clip]{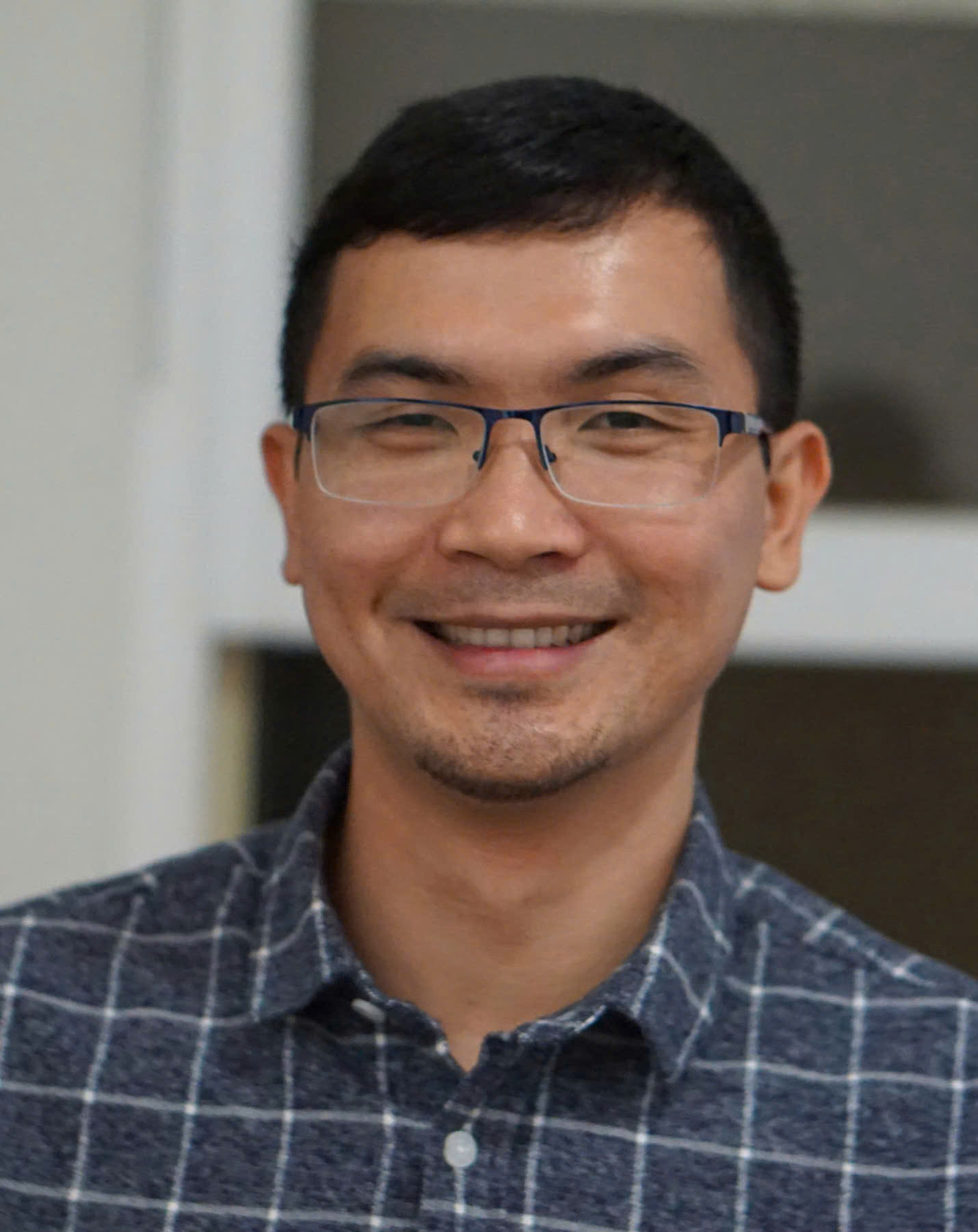}}]{Manh Duong Phung} received his Ph.D. from Vietnam National University, Hanoi, in 2015 and completed a postdoctoral fellowship at the University of Technology Sydney in 2017. He is currently a Senior Lecturer at Fulbright University Vietnam. His research interests include unmanned aerial vehicles, autonomous robots, and swarm intelligence. He has served on the organizing committees of various international conferences and as a reviewer for reputable journals published by IEEE, Elsevier, and Springer. He is the recipient of several research awards, including the Endeavour Research Fellowship from the Australian Government and the Best Paper Award at the Australasian Conference on Robotics and Automation.
\end{IEEEbiography}





\end{document}